\documentclass[preprint,12pt]{elsarticle}

\usepackage{amssymb}

\usepackage{amsmath}

\journal{Information Fusion}
\usepackage{hyperref} 
\usepackage{algorithm}
\usepackage{algorithmic}
\usepackage{amsthm} 
\usepackage{xspace}
\newcommand{\ours}{\textsc{DAMPE}\xspace}

\usepackage{amsmath,amsfonts,bm}

\def\eqref#1{equation~\ref{#1}}
\def\1{\bm{1}}

\DeclareMathAlphabet{\mathsfit}{\encodingdefault}{\sfdefault}{m}{sl}
\SetMathAlphabet{\mathsfit}{bold}{\encodingdefault}{\sfdefault}{bx}{n}

\newtheorem{proposition}{Proposition}

\usepackage{booktabs}
\usepackage{multicol}
\usepackage{multirow}
\usepackage{amssymb}
\usepackage{bbding}
\usepackage{graphicx}    
\usepackage{subcaption}  
\begin{document}

\begin{frontmatter}

\title{A Novel Framework for Multi-Modal Protein Representation Learning}

\author[1]{Runjie Zheng} %% Author name
\ead{zhengrj26@mail2.sysu.edu.cn}

\author[1]{Zhen Wang\corref{cor1}}
\ead{wangzh665@mail.sysu.edu.cn}

\author[1]{Anjie Qiao}
\ead{qiaoanj@mail2.sysu.edu.cn}

\author[1]{Jiancong Xie}
\ead{xiejc3@mail2.sysu.edu.cn}

\author[1]{Jiahua Rao}
\ead{raojh7@mail.sysu.edu.cn}

\author[1]{Yuedong Yang}
\ead{yangyd25@mail.sysu.edu.cn}
\cortext[cor1]{Corresponding author: Zhen Wang. E-mail: wangzh665@mail.sysu.edu.cn}

%% Author affiliation
\affiliation[1]{
organization={School of Computer Science and Engineering, Sun Yat-sen University (SYSU)},
addressline={No. 132, Outer Ring East Road, University Town, Panyu District},
city={Guangzhou},
postcode={510006},
state={Guangdong},
country={China}}

%% Abstract
\begin{abstract}
Accurate protein function prediction requires integrating heterogeneous intrinsic signals (e.g., sequence and structure) with noisy extrinsic contexts (e.g., protein–protein interactions and GO term annotations).
However, two key challenges hinder effective fusion: (i) cross-modal distributional mismatch among embeddings produced by pre-trained intrinsic encoders, and (ii) noisy relational graphs of extrinsic data that degrade GNN-based information aggregation.
We propose \textbf{D}iffused and \textbf{A}ligned \textbf{M}ulti-modal \textbf{P}rotein \textbf{E}mbedding (\textbf{DAMPE}), a unified framework that addresses these through two core mechanisms.
First, we propose Optimal Transport (OT)-based representation alignment that establishes correspondence between intrinsic embedding spaces of different modalities, effectively mitigating cross-modal heterogeneity.
Second, we develop a Conditional Graph Generation (CGG)-based information fusion method, where a condition encoder fuses the aligned intrinsic embeddings to provide informative cues for graph reconstruction.
Meanwhile, our theoretical analysis implies that the CGG objective drives this condition encoder to absorb graph-aware knowledge into its produced protein representations.
Empirically, \ours outperforms or matches state-of-the-art methods such as DPFunc on standard GO benchmarks, achieving AUPR gains of 0.002–0.013 pp and Fmax gains 0.004–0.007 pp.
Ablation studies further show that OT-based alignment contributes 0.043–0.064 pp AUPR, while CGG-based fusion adds 0.005–0.111 pp Fmax.
Overall, \ours offers a scalable and theoretically grounded approach for robust multi-modal protein representation learning, substantially enhancing protein function prediction.
\end{abstract}

\begin{keyword}
Protein Function Prediction \sep Multi-modal Representation Learning
\sep Optimal Transport \sep  Conditional Graph Generation 

\end{keyword}

\end{frontmatter}

\section{Introduction}
\label{sec:intro}
Proteins are essential biomolecules whose diverse structures and functions underpin key biological processes, including enzymatic activity, signal transduction, transport, and structural support~\cite{Hirokawa2009,Ferrell2000,kim_intermediate_2007}. 
Predicting protein function at scale is crucial for understanding biology and driving applications in drug design and disease research~\cite{Jumper_Alphafold2021,m_how_2023}, where accurate prediction strongly depends on high-quality protein representations that can capture function-relevant biological features.

Deep learning has advanced protein representation learning for function prediction. Pre-trained protein language models (PLMs)~\cite{Brandes2021,ESM_Rives2021,elnaggar2021prottrans} generate scalable sequence embeddings that outperform homology-based tools such as BLASTKNN~\cite{blastknn_cafa} and DIAMOND~\cite{diamond}. Yet PLMs lack structural and interaction context. Structure-based models like GVP~\cite{gvp_Jing2021} and GearNet~\cite{zhang2022gearnet}, empowered by AlphaFold2~\cite{Jumper_Alphafold2021}, capture geometric patterns but overlook evolutionary relations. These complementary strengths motivate integrating sequence and structure embeddings to enhance protein function prediction.

However, effectively merging these two modalities remains challenging. Naive concatenation suffers from cross-modal mismatch due to distinct feature geometries and distributions, whereas full joint re-training (e.g., cascading the sequence encoder into the structure encoder) greatly increases computational cost and parameters~\cite{Wang2022LMGVP,zhang2023systematic}. Moreover, theoretical studies on multi-modal learning~\cite{Modalbias} show that data imbalance, such as abundant, robust sequence features versus sparse structural ones, induces modal dominance, causing the model to over-rely on sequences and weaken structural representation learning.

To mitigate these issues, recent studies have begun exploring alignment of sequence- and structure-derived embeddings through lightweight fine-tuning of pre-trained encoders~\cite{SPLM}, often drawing inspiration from cross-modal contrastive learning in vision–language models~\cite{CLIP}.

Although conceptually appealing, cross-modal contrastive learning faces two major issues in the protein domain. (1) Restricted positives: Experimental or predicted structures represent only one static conformation, so each sequence–structure pair is treated as a unique ``true positive'', ignoring proteins’ natural conformational plasticity~\cite{10.7554/eLife.75751}. (2) False negatives: High sequence homology and fold convergence cause functionally similar proteins to be mislabeled as negatives, biasing the objective and limiting generalization~\cite{10.1093/nargab/lqac043}.

Furthermore, sequence and structure are regarded as intrinsic information of proteins because they directly encode the molecule itself. 
Beyond these intrinsic descriptors, extrinsic biological contexts, such as protein–protein interaction (PPI) networks~\cite{Szklarczyk2023_STRING} and Gene Ontology (GO) annotations~\cite{GeneOntologyConsortium2023}, offer complementary information that reflects the functional interplay and contextual roles of proteins.
Thus, many function-prediction methods~\cite{hou2021tale,DeepGO_2017,DEEPGRAPHGO,DPFunc_NatCommun2024} incorporate such extrinsic information.

A common strategy is to apply graph neural networks (GNNs) over PPI networks~\cite{hou2021tale,DeepGO_2017,DEEPGRAPHGO}, using intrinsic features as initial node features to integrate both intrinsic and extrinsic information in the eventual node embeddings.
Although effective, these approaches face two key limitations.
(1) They rely on message passing, which is inherently vulnerable to noisy or incomplete PPI edges~\cite{PPINosiy}; such noise can propagate and amplify through the network, reducing robustness, particularly in sparse or weakly annotated graphs~\cite{pmlr-v202-dong23a}.
(2) Their adopted link prediction objective assumes independence assumes conditional independence among edges, which is often violated in practical biological graphs due to missing links and the limited receptive field of GNNs.

To address existing limitations that stem from inter-modal heterogeneity and inherent flaws of extrinsic biological contexts, we propose \textbf{D}iffused and \textbf{A}ligned \textbf{M}ulti-modal \textbf{P}rotein \textbf{E}mbedding (\textbf{DAMPE}), a unified framework that integrates intrinsic and extrinsic information through two key mechanisms.

First, to alleviate cross-modal heterogeneity, we propose \textbf{Optimal Transport (OT)-based representation alignment}, which projects structural embeddings into the sequence embedding space, while enabling the reuse of frozen pre-trained encoders, thus substantially cutting retraining costs. By contrast, while contrastive learning is widely used for multi-modal alignment, it has well-documented drawbacks in biological applications (as detailed earlier); our OT-based approach circumvents these issues entirely, enabling more robust representation alignment.

Second, \textbf{Conditional Graph Generation (CGG)-based information fusion}: Rather than directly performing message passing on noisy a PPI
network, we train a conditional diffusion model to estimate the distribution of a heterogeneous graph's edge types conditioned on the proteins' intrinsic descriptors.
Specifically, we train a denoising network to reconstruct the clean heterogeneous graph from a noisy one, with a condition encoder that fuses the aligned embeddings of protein nodes into integrated condition embedding.
This embedding is then fed into the denoising network to guide the reconstruction process.
By optimizing the generative objective, this condition encoder is updated by the gradients propagating through the denoising network, which drive the condition encoder to internalize relational structure while remaining robust to graph noise~\cite{Graffe}.

Importantly, this condition encoder, learned through CGG training, then offers protein representation for function prediction. Since we parameterize it by a light-weight Mixture-of-Experts architecture, and graph-aware knowledge has been injected via CGG, we eliminate the need for traditional GNNs' iterative message passing.
By discarding such computational cost, it avoids inference latency, a benefit further validated by our experiments.

Our main contributions are summarized as follows:
\begin{itemize}
  \item We present \ours, a multi-modal protein representation learning framework that integrates intrinsic (sequence and structure) and extrinsic (PPI and GO) features for enhanced function prediction.
  \item To our knowledge, \ours is the first to leverage OT and CGG for protein function prediction, and our theoretical analysis explains what CGG drives the condition encoder to learn.
  \item Extensive experiments show that \ours outperforms baselines on most metrics in downstream tasks and achieves significant efficiency improvements. Furthermore, comprehensive ablation studies confirm the indispensable contribution of each mechanism.
\end{itemize}
\section{Related work}
\label{sec:related}
In this section, we present a brief summary of protein representation learning methods and their application in protein function prediction.

\noindent \textbf{Sequence encoders}
Sequence encoders have evolved significantly to capture discriminative features from amino acid sequences. Early approaches leveraged Convolutional Neural Network (CNN) to extract local motifs~\cite{CNN_Shanehsazzadeh}, while Long Short-Term Memory (LSTM) networks addressed long-range dependencies, with hybrid CNN-LSTM models further improving performance~\cite{LSTM_wang2016protein}. However, the most impactful breakthroughs came with Transformer-based Protein Language Models (PLMs), such as ESM~\cite{ESM_Rives2021} and ProtTrans~\cite{elnaggar2021prottrans}, which learn contextual embeddings from massive sequence datasets and dominate most downstream tasks. 

\noindent \textbf{Structure encoders}
Structure encoders capture the spatial and geometric properties of protein conformations and are commonly classified by their basic geometric primitive: residue-level models such as GearNet~\cite{zhang2022gearnet} and DeepFRI~\cite{deepfri_gligorijevic2021structure}; atom-level methods that operate on atomic point clouds or via 3D convolutions and thus offer higher chemical granularity at greater computational cost (e.g., IEConv~\cite{hermosilla2021ieconv}); and surface-based encoders that represent molecular surfaces as meshes or surface patches to emphasize interface geometry and chemistry (e.g., dMaSIF~\cite{dmasif_sverrisson2021fast}).
Residue-level models typically treat residues as graph nodes with edges defined by spatial proximity or biochemical contacts, whereas atom-level and surface-based approaches encode finer geometric/chemical detail at the cost of increased computation and data requirements.
Recent self-supervised pretraining (contrastive and diffusion-based) has further improved the transferability of these geometric representations~\cite{hermosilla2022contrastive,xu2023pretraining}. 
% list and cite related works

\noindent \textbf{Multi-modal Representation Learners} Multi-modal approaches integrate sequence and structure information to overcome the limitations of single-modality methods. LM-GVP~\cite{Wang2022LMGVP} combines PLMs with geometric vector perceptrons to infuse structural context into sequence embeddings, improving performance on stability prediction tasks. ESM-GearNet~\cite{zhang2023systematic} systematically compared three fusion paradigms between ESM-2 and structural encoders~\cite{gvp_Jing2021,zhang2022gearnet,cdconv_Fan2023}: \textit{Serial} fusion, where sequence representations are injected as residue features into the structure encoder; \textit{Parallel} fusion, where sequence and structure embeddings are concatenated; and \textit{Cross} fusion, which integrates the two modalities via multi-head self-attention. 
The study finds that serial fusion—i.e., augmenting geometric models with PLM-derived residue features—yields the best performance on various downstream tasks.
Similarly, SST-ResNet~\cite{zhang2023sst} employs multi-scale fusion to integrate sequence and structure, outperforming single-modality models in property prediction.
Notably, while serial fusion demonstrates superior performance, its practical scalability is constrained by heavy computational overhead—integrating large-scale PLMs with structural encoders often results in oversized model architectures, making full joint training computationally prohibitive, especially for large datasets or resource-limited settings.

\noindent \textbf{Protein function prediction}
Protein function prediction has benefited from these representation advances.
Transformer-based methods like TALE~\cite{hou2021tale} jointly embed sequences and hierarchical function labels, while ATGO~\cite{zhang2022atgo} combines PLMs with triplet networks for accurate Gene Ontology (GO) prediction. DPFunc~\cite{DPFunc_NatCommun2024} integrates sequence, structure, and domain information via cross-attention to enhance interpretability and accuracy.
These methods highlight the growing importance of rich, multi-source representations in advancing functional annotation.
In contrast, our work focuses on integrating protein-protein interaction (PPI) networks and GO annotations through conditional graph generation-based information fusion, offering a novel approach to leverage both interaction and functional knowledge for more comprehensive and accurate protein function prediction.
\section{Problem Formulation}
\label{sec:problem}
Gene Ontology (GO) is the standard resource for characterizing protein function, and is often organized in three aspects: Molecular Function (MF), Biological Process (BP), and Cellular Component (CC). In this work, protein function prediction is framed as a multi-class multi-label classification task targeting GO terms.

To address this task, we design a multi-modal representation learning framework that fuses two types of complementary information into a unified protein representation: (1) intrinsic information, namely protein-specific sequence and structural features; and (2) extrinsic information, namely a heterogeneous graph composed of a Protein-Protein Interaction (PPI) network and a Gene Ontology (GO).
These two kinds of information, along with their integration details, are elaborated as follows.

\subsection{Intrinsic Information}
The protein dataset \(\mathcal{P} = \{ p_1, \dots, p_{N_p} \}\) comprises \(N_p\) proteins, each characterized by its amino acid sequence and 3D backbone structure represented as a residue-level geometric graph.
Two complementary embeddings encode these data, respectively:
\begin{itemize}
    \item \textit{Sequence embeddings} \(\mathbf{E}^{\text{seq}} \in \mathbb{R}^{N_p \times d_{\text{seq}}}\) are produced from pre-trained protein language model (PLM).
    \item \textit{Structural embeddings} \(\mathbf{E}^{\text{struc}} \in \mathbb{R}^{N_p \times d_{\text{struc}}}\) are
 representations extracted from 3D backbone geometric graph by structure encoders.
\end{itemize}
As a general framework, \ours allows flexible selection of these encoders.
In this work, we specifically adopt ESM-1b~\cite{ESM_Rives2021} and GearNet~\cite{zhang2022gearnet} as the PLM for sequence embeddings and structure encoder for structural embeddings, respectively.

\subsection{Extrinsic Information}  
We consider two types of entities for modeling extrinsic biological knowledge: proteins \(\mathcal{P}\) and GO terms \(\mathcal{O}^{(o)} = \{\tau_1, \dots, \tau_{N_o}\}\) of each specific ontology $o\in\{\mathrm{MF,BP,CC}\}$.
To capture their topological relations, we define four relation types: \(\mathcal{R}=\{r_{\mathrm{ppi}},\,r_{\mathrm{go}},\,r_{\mathrm{anno}},\,r_{\varnothing}\},\) where:
%\(r_{\mathrm{ppi}}\) denotes protein–protein interactions, \(r_{\mathrm{go}}\) denotes GO hierarchical parent–child links (e.g., \texttt{is\_a}, \texttt{part\_of}), \(r_{\mathrm{anno}}\) denotes protein–GO annotation links, and \(r_{\varnothing}\) denotes the absence of an edge.
\begin{enumerate}
  \item \textbf{PPI edges} ($r_{\mathrm{ppi}}$). We obtain PPI from STRING v12.0~\cite{Szklarczyk2023_STRING} and retain only high-confidence associations with combined score $\ge 700$. From the full STRING dump we select only edges whose both endpoints belong to $\mathcal{P}$; duplicate edges and self-loops are removed.
  \item \textbf{GO hierarchical edges} ($r_{\mathrm{go}}$). The GO DAG~\cite{GeneOntologyConsortium2023} is parsed from the same data release used by DPFunc. We keep primary hierarchical relations (e.g., \texttt{is\_a}, \texttt{part\_of}) and remove obsolete/alt terms; parent–child links between terms in $\mathcal{O}^{(o)}$ are added as $r_{\mathrm{go}}$ edges.
  \item \textbf{Protein–GO annotation edges} ($r_{\mathrm{anno}}$). For each protein $p\in\mathcal{P}$ we connect $p$ to every GO term $\tau\in\mathcal{O}^{(o)}$ that is among $p$'s annotations in the DPFunc label set. \textbf{To avoid label leakage}, annotation edges incident to proteins in the official test split are \textbf{not} included in $\mathcal{G}^{(o)}$.
  \item \textbf{No-edge relations} ($r_{\varnothing}$). For all pairs of nodes $(u,v)$ in $\mathcal{V}^{(o)} \times \mathcal{V}^{(o)}$ where no explicit edge exists, we explicitly assign the $r_{\varnothing}$ relation type. This universal no-edge type simplifies modeling while preserving the distinction between existing and absent connections in the graph.
\end{enumerate}

Putting these together, we consider each ontology-specific heterogeneous graph $\mathcal{G}^{(o)}=(\mathcal{V}^{(o)},\mathcal{E}^{(o)},\nu,\rho)$.
$\mathcal{V}^{(o)}=\mathcal{P}\cup\mathcal{O}^{(o)}$, where $\mathcal{O}^{(o)}\subset\mathcal{O}$ contains GO terms belonging to ontology $o$ that are candidate labels in the DPFunc~\cite{DPFunc_NatCommun2024} dataset. Edges in $\mathcal{E}^{(o)}$ comprise those four types.
The mapping $\nu:\mathcal{V} \to \{\text{p},\text{go}\}$ indicates the node type (protein or GO term), and \(\rho:\mathcal{E}\to\mathcal{R}\) assigns each edge its relation type. An edge is therefore represented by a triplet \((u,v,r)\in\mathcal{E}\) with \(\rho(u,v)=r\).
The relation-aware topology of a \(\mathcal{G}^{(o)}\) is encoded as a binary adjacency tensor $\mathbf{A}\in\{0,1\}^{|\mathcal{V}|\times|\mathcal{V}|\times|\mathcal{R}|},$
with entries
\[
\mathbf{A}_{i,j,r}=
\begin{cases}
1, & \text{if }(i,j)\in\mathcal{E}^{(o)}\text{ and }\rho(i,j)=r\\
0, & \text{otherwise}
\end{cases},
\]
explicitly recording which relation \(r\in\mathcal{R}\) holds for each node pair \((i,j)\).

To furnish node features for each \(\mathcal{G}^{(o)}\), GO-term priors are provided by Poincaré (hyperbolic) embeddings~\cite{poincare} $\mathbf{Z}\in\Re^{N_o\times d_{\mathrm{GO}}}$,
which are trained independently for each ontology and supply hierarchical priors that guide the learning of relation-aware protein representations.

\begin{figure}[t]
\centering 
\includegraphics[width=0.7\linewidth]{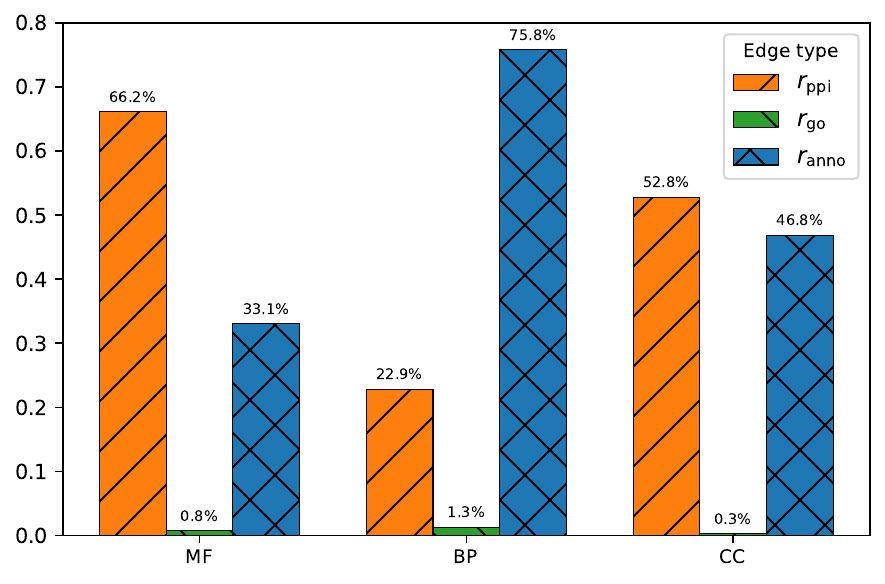}
% \caption{Fraction of Edge Types Across Different Ontology. The proportions of edge types are normalized, excluding isolated nodes to highlight the relative prevalence of each edge type among connected nodes.}
\caption{Fraction of Edge Types Across Different Ontology. The proportions of edge types are normalized, with \(r_{\varnothing}\) accounting for 99.91\% in MF, 99.81\% in BP, and 99.87\% in CC. Isolated nodes are excluded to highlight the relative prevalence of each edge type among connected nodes.}
\label{fig:edge_types}
\end{figure}

%In this work, we construct ontology-specific heterogeneous graphs $\mathcal{G}^{(o)}$ for each ontology $o\in\{\mathrm{MF,BP,CC}\}$. Each $\mathcal{G}^{(o)}=(\mathcal{V}^{(o)},\mathcal{E}^{(o)},\nu,\rho)$ is defined with $\mathcal{V}^{(o)}=\mathcal{P}\cup\mathcal{O}^{(o)}$, where $\mathcal{O}^{(o)}\subset\mathcal{O}$ contains GO terms belonging to ontology $o$ that are candidate labels in the DPFunc~\cite{DPFunc_NatCommun2024} dataset.  Edges in $\mathcal{E}^{(o)}$ comprise four types:

To summarize, \(\mathcal{G}^{(o)}\) integrate proteins and ontology-specific GO terms via four edge types as defined.
Notably, these graphs remain highly sparse due to the scarcity of explicit interactions relative to all possible node pairs, a characteristic reflected in the distribution of edge types shown in Figure~\ref{fig:edge_types}.
\section{Proposed Method}
\label{sec:method}
\begin{figure*}[t] 
  \centering  \includegraphics[width=\linewidth]{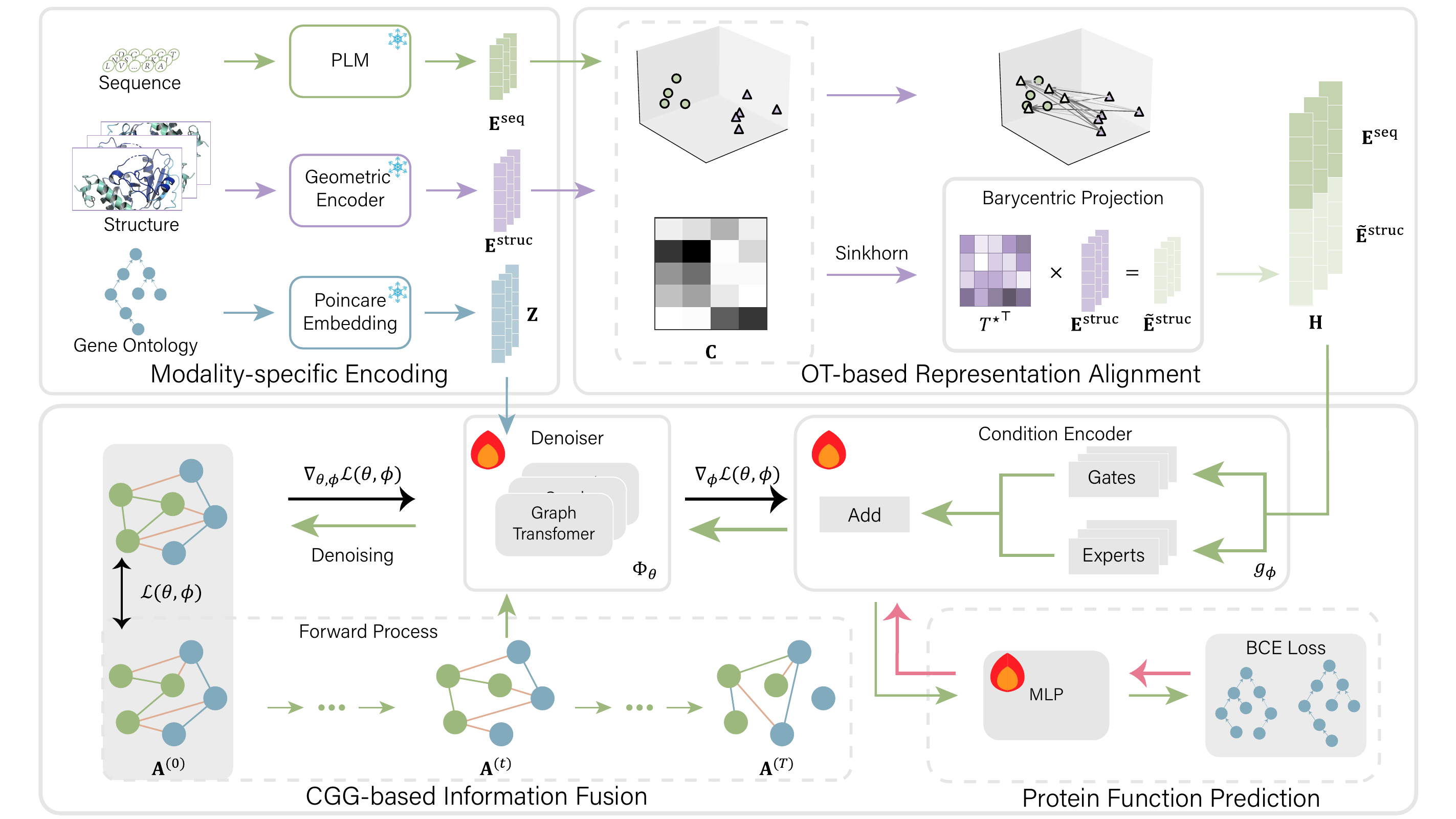}
  \caption{Pipeline of \ours, which adopts Optimal Transport (OT)-based representation alignment and Conditional Graph Generation (CGG)-based information fusion to learn robust and versatile protein representations for function prediction.}
  \label{fig:framework}
\end{figure*}

To address the key challenges outlined in Section~\ref{sec:problem}, we propose \ours, a multi-modal protein representation learning framework.
As illustrated in Figure~\ref{fig:framework}, \ours consists of four interconnected stages: First, in the \textbf{Modality-Specific Encoding} stage, we generate pre-trained representations for each modality: sequence, structure, and GO terms.
Optimal Transport (OT)-based representation alignment is then applied to sequence and structural embeddings, while GO embeddings are reserved for the subsequent information fusion step.
Second, in \textbf{OT-based representation alignment}, we align structural embeddings with the sequence embeddings space via solving an OT problem, mitigating cross-modal heterogeneity and laying a homogeneous basis for subsequent fusion.
Third, \textbf{CGG-based information fusion} integrates intrinsic and extrinsic information via conditional generative modeling of the PPI–GO heterogeneous graph, with a condition encoder learning to fuse the intrinsic features while absorbing the graph-aware knowledge from the backbone of the generative model.
Fourth, for \textbf{Protein Function Prediction}, the learned comprehensive multi-modal embeddings are fed to a classifier, which is trained to predict the probabilities for each GO term.

\subsection{Optimal Transport-based Representation Alignment}
Instead of re-training pre-trained encoders (which is costly and risks erasing learned knowledge), we leverage their fixed outputs and resolve heterogeneity via entropy-regularized optimal transport (OT)~\cite{Cuturi2013_Sinkhorn}. This procedure globally aligns the two modalities by projecting structural embeddings into the sequence space (using robust PLM representations as a reference), yielding homogeneous intrinsic features for downstream fusion. Notably, this design preserves pre-trained encoders’ strengths while avoiding contrastive biases (e.g., unreliable positives/negatives).

Specifically, the cost matrix for OT is derived from the frozen sequence and structure embeddings by taking the root-mean-square error (RMSE) between each structural dimension \(i\) and each sequence dimension \(j\) across all proteins:
\(
\mathbf{u}_i = \mathbf{E}^{\text{struc}}_{\colon,i} \in \mathbb{R}^{N_p}, \quad 
\mathbf{v}_j = \mathbf{E}^{\text{seq}}_{\colon,j} \in \mathbb{R}^{N_p},
\)
where $\mathbf{u}_i$ and $\mathbf{v}_j$ denote the $i$-th column of the structure embeddings and the $j$-th column of the sequence embeddings, respectively.
Then the cost matrix $\mathbf{C} \in \mathbb{R}^{d_{\text{struc}} \times d_{\text{seq}}}$ over the full protein set is defined as follows:
\begin{equation}
\mathbf{C}_{ij} = \sqrt{\frac{1}{N_p} \sum_{p=1}^{N_p} \left( \mathbf{E}^{\text{struc}}_{p,i} - \mathbf{E}^{\text{seq}}_{p,j} \right)^2},
\end{equation}
The RMSE formulation therefore yields a robust, bounded per-dimension discrepancy measure and mitigates numerical overflow from large squared differences. 

Treating each component as an sample from $\Re^{N_{\mathrm{P}}}$, we place uniform mass over these samples and define the corresponding (empirical) marginal measures for structure and sequence embedding components, respectively:
\(
\alpha(\cdot) = \tfrac{1}{d_{\mathrm{struc}}}\sum_{i=1}^{d_{\mathrm{struc}}}\delta_{\mathbf{u}_i}(\cdot), \quad
\beta(\cdot) = \tfrac{1}{d_{\mathrm{seq}}}\sum_{j=1}^{d_{\mathrm{seq}}}\delta_{\mathbf{v}_j}(\cdot),
\) 
where $\delta_{\mathbf{x}}(\cdot)$ denotes the Dirac delta measure concentrated at point $\mathbf{x}$.
 
With these (empirical) marginal measures defined, we solve the entropy-regularized OT problem:
\begin{equation}
 T^\star = \mathop{\mathrm{arg}\,\mathrm{min}}_{T \in \Pi(\alpha, \beta)} \langle T, \mathbf{C} \rangle + \varepsilon\,\mathcal{H}(T),   
\end{equation}
where $\varepsilon>0$ is the entropy-regularization coefficient controlling the trade-off between fidelity to the unregularized OT objective and solution smoothness (smaller $\varepsilon$ approaches the Wasserstein solution, larger $\varepsilon$ yields smoother, higher-entropy transport plans), $\mathcal{H}(T)=-\sum_{i,j}T_{ij}\log T_{ij}$ is the entropy term, and $\Pi(\alpha,\beta)=\left\{T\in\mathbb{R}_+^{d_{\mathrm{struc}}\times d_{\mathrm{seq}}} 
\;\middle|\; 
T\mathbf{1}_{d_{\mathrm{seq}}}=\frac{1}{d_{\mathrm{struc}}}\mathbf{1}_{d_{\mathrm{struc}}},\quad 
T^\top\mathbf{1}_{d_{\mathrm{struc}}}=\frac{1}{d_{\mathrm{seq}}}\mathbf{1}_{d_{\mathrm{seq}}} 
\right\}$ is the transport polytope, where $\mathbf{1}_d$ denotes the all-one vector in $\mathbb{R}^d$.

The resulting transport plan \(T^\star\) is computed once from the full dataset and then used to align individual protein embeddings via barycentric projection: $\tilde{\mathbf{E}}^{\mathrm{struc}} \;=\; \mathbf{E}^{\mathrm{struc}}{T^\star} \in \mathbb{R}^{N_p \times d_{\mathrm{seq}}}$.
Here the mapped vectors, the rows of \(\tilde{\mathbf{E}}^{\mathrm{struc}}\) are thus embedded into the sequence embedding space and can be more effectively fused with \(\mathbf{E}^{\mathrm{seq}}\) in downstream stages.
Then, the intrinsic protein representation is obtained by concatenating the aligned representations of these two modalities:
\begin{equation}
\mathbf{H} = \left[\mathbf{E}^{\mathrm{seq}}; \tilde{\mathbf{E}}^{\mathrm{struc}}\right]\in \mathbb{R}^{N_p\times 2d_{\mathrm{seq}}}.    
\end{equation}

\subsection{Conditional Graph Generation-based Information Fusion}

For each protein node in our heterogeneous graph, we can induce its $k$-hop ego-graph. For brevity, we reuse $\mathbf{A}$, $\mathbf{H}$, and $\mathbf{Z}$ to denote the binary adjacency tensor, the embeddings of protein nodes, and the embeddings of GO-term nodes within this ego-graph, respectively.
Each protein's ego-graph can be regarded as an i.i.d. sample drawn from an underlying data distribution $q(\mathbf{A}\mid\mathbf{H},\mathbf{Z})$.
In \ours, \textit{conditional graph generation (CGG)} refers to estimating this conditional distribution.

Motivated by the recent success of graph diffusion models, we adapt DiGress~\cite{DiGress2022} to estimate \(q(\mathbf{A}\mid\mathbf{H},\mathbf{Z})\).
It comprises a forward (diffusion) process that progressively perturbs the adjacency tensor and a reverse (denoising) process that refines it.
In our formulation, $\mathbf{H}$ is first fed into a condition encoder to produce condition embeddings $\tilde{\mathbf{H}}$. These embeddings, together with $\mathbf{Z}$, are then fed into a Graph Transformer-based denoiser, which takes advantage of the embeddings to better reconstruct \(\mathbf{A}\).
By jointly optimizing the generative objective w.r.t. both the denoiser and the condition encoder, the latter is encouraged to distill informative cues from the input intrinsic features $\mathbf{H}$ that enhance the prediction of corresponding extrinsic relationships, as further supported by our theoretical analysis in Section~\ref{sec:analysis}.

\subsubsection{Forward process}
To account for the uncertainty inherent in the sparse and noisy adjacency tensor \(\mathbf{A}\), we adopt the Markovian discrete noising process proposed from DiGress. This process defines a sequence of discrete random variables $\mathbf{A}^{(t)}, t=0,\ldots,T$, representing progressively corrupted edge-type information. The boundary conditions are specified such that $\mathbf{A}^{(0)}$ follows the true data distribution $q$, while $\mathbf{A}^{(T)}$ conforms to the empirical distribution of relation types observed in the training data.

For each node pair \((u,v)\) in an ego-graph, let \(\mathbf{A}_{u,v}^{(t)} \in \{0,1\}^{|\mathcal{R}|}\) denote the one-hot row vector at timestep \(t\) for denoting the edge-type, where \(\mathbf{A}_{u,v,r}^{(t)} = 1\) iff edge \((u,v)\) has relation type \(r\) at \(t\).
We define per-step transition matrices \(\mathbf{Q}^{(t)} \in \Re^{|\mathcal{R}| \times |\mathcal{R}|}\) (row-stochastic) for $t=1,\ldots,T$, as $\mathbf{Q}^{(t)} \triangleq \alpha_t \mathbf{I} + (1 - \alpha_t) \mathbf{1} \mathbf{m}^\top
$, where \(\mathbf{m}\in\Re^{|\mathcal{R}|}\) denotes the marginal distribution over relation types \(\mathcal{R}\) in the training data, $\mathbf{I}$ is the $|\mathcal{R}|$-dimensional identity matrix (used to retain the probability of preserving the original relation type at each noising step), $\mathbf{1}$ denotes a \(|\mathcal{R}|\)-dimensional column vector with all elements equal to 1, and $\alpha_t$ follows a shifted cosine schedule~\cite{diffusion_cosine}.

Then, the cumulative transition kernel becomes $\bar{\mathbf{Q}}^{(t)} = \prod_{\tau=1}^t \mathbf{Q}^{(\tau)} = \bar{\alpha}_t \mathbf{I} + (1 - \bar{\alpha}_t) \mathbf{1} \mathbf{m}^\top
$, with $\bar{\alpha}_t = \prod_{\tau=1}^t \alpha_\tau$. This yields the closed-form transition distribution: $q\bigl(\mathbf{A}_{u,v}^{(t)} \mid \mathbf{A}_{u,v}^{(0)}\bigr) = \mathbf{A}_{u,v}^{(0)} \bar{\mathbf{Q}}^{(t)}$.
% \[
% q\bigl(\mathbf{A}_{u,v}^{(t)} \mid \mathbf{A}_{u,v}^{(0)}\bigr) = \mathbf{A}_{u,v}^{(0)} \bar{\mathbf{Q}}^{(t)}.
% \]
The joint conditional distribution over all edges in the graph is then given by the product of the individual edge-type distributions, assuming conditional independence across different node pairs \((u,v)\):
\begin{equation}
q\bigl(\mathbf{A}^{(t)} \mid \mathbf{A}^{(0)}\bigr)
\triangleq \prod_{(u,v) \in \mathcal{V}\times \mathcal{V}}
q\bigl(\mathbf{A}_{u,v}^{(t)} \mid \mathbf{A}_{u,v}^{(0)}\bigr).
\label{eq:forward}
\end{equation}
As \(t \to \infty\), \(\bar{\alpha}_t \to 0\) and \( (1 - \bar{\alpha}_t) \to 1\), so that each edge-type distribution converges to the marginal distribution \(\mathbf{m}\).

\subsubsection{Reverse process}  
The forward process corrupts clean $\mathbf{A}$ into noisy states, and the reverse process aims to invert this corruption to recover biologically meaningful edge types.
To this end, we follow DiGress to approximate $q(\mathbf{A}^{(t-1)}\mid \mathbf{A}^{(t)})$ by $q(\mathbf{A}^{(t-1)}\mid \mathbf{A}^{(t)},\mathbf{A}^{(0)}) \propto q(\mathbf{A}^{(t)}\mid \mathbf{A}^{(t-1)})q(\mathbf{A}^{(t-1)}\mid \mathbf{A}^{(0)})$.
However, as the clean sample $\mathbf{A}^{(0)}$ is unavailable in advance during sampling, we need to predict it based on the current noisy state.
Therefore, we parameterize the reverse transition distribution \(p_{\theta}(\mathbf{A}^{(t-1)} \mid \mathbf{A}^{(t)}) \triangleq \sum_{\mathbf{A}^{(0)}} p_{\theta}(\mathbf{A}^{(0)} \mid \mathbf{A}^{(t)}) q(\mathbf{A}^{(t-1)} \mid \mathbf{A}^{(t)},\mathbf{A}^{(0)}) \) using a denoising network to estimate $p_{\theta}(\mathbf{A}^{(0)} \mid \mathbf{A}^{(t)})$.

Critically, our setting extends this formulation to CGG, that is, we condition \(p_\theta(\mathbf{A}^{(t-1)} \mid \mathbf{A}^{(t)})\) (with \(\theta\) denoting the parameters of the denoising network) on intrinsic protein embeddings \(\tilde{\mathbf{H}}\) and extrinsic GO-term embeddings \(\mathbf{Z}\).
Thus, we adopt classifier-free guidance~\cite{ho2021_CFG} to estimate $p_{\theta,\phi}(\mathbf{A}^{(0)} \mid \mathbf{A}^{(t)}, \mathbf{H}, \mathbf{Z})$.
Specifically, we design a condition encoder $g_{\phi}$ to produce condition embedding $\tilde{\mathbf{H}}$ from the given intrinsic features $\mathbf{H}$, as we will detail later.
The denoising network \(\Phi_{\theta}\) (implemented as a Graph Transformer~\cite{graphTransfomre}) takes three inputs: the noisy adjacency tensor \(\mathbf{A}^{(t)}\), the conditioning signals \(\mathcal{C}=\{\tilde{\mathbf{H}},\mathbf{Z}\} \), and the diffusion timestep \(t\).
%Here \(\tilde{\mathbf{H}}=g_{\phi}(\mathbf{H}_{p})\) denotes the per-node context embeddings produced by the MoE module (see Eq.~\ref{eq:H}); \(\mathbf{Z}\) are the fixed, pretrained GO embeddings. The parameters \(\phi\) and \(\theta\) denote denoiser weights. For notational brevity we write \(\mathcal{C}=\mathcal{C}(\phi)\) when the dependence on \(\phi\) is implicit.
In our adopted Graph Transformer architecture, $\Phi_{\theta}$ treats $\mathcal{C}$ as the input node tokens, $\mathbf{A}^{(t)}$ as the input edge tokens, and $t$ as the graph-level token.
Its outputs on edge tokens encode the parameterized probability distribution, denoted by:
\begin{equation}
\widehat{\mathbf{P}}^{(t)} = \Phi_{\theta}\left(\mathbf{A}^{(t)},\, \mathcal{C},\, t\right) \in [0,1]^{|\mathcal{V}|\times|\mathcal{V}|\times|\mathcal{R}|},    
\label{eq:condpred}
\end{equation}
where \(\widehat{\mathbf{P}}^{(t)}_{i,j,r}\) indicates \(p_{\theta,\phi}\left(\mathbf{A}_{i,j,r}^{(0)}=1 | \mathbf{A}^{(t)}, \mathbf{H}, \mathbf{Z} \right)\).
To implement classifier-free guidance, we randomly drop \(\mathcal{C}\) during the training, allowing the model to simultaneously learn both conditional and unconditional (with/without \(\mathcal{C}\)) probability distribution of \(\mathbf{A}^{(0)}\).

\noindent\textbf{Details of condition encoder.} Reverse denoiser requires a compact, informative conditioning signal that summarizes intrinsic protein descriptions. 
To provide this, we employ a per-node Mixture-of-Experts (MoE) as the condition encoder: 
for each protein $p$, the MoE consumes the OT-aligned intrinsic input \(\mathbf{H}_p\) and produces the fused context embedding \(\tilde{\mathbf{H}}_p\); 
the denoiser then takes \(\tilde{\mathbf{H}}\) (encoding intrinsic protein patterns) and \(\mathbf{Z}\) (capturing extrinsic functional knowledge) as guiding conditions to steer its recovery of biologically meaningful edge types in \(\mathbf{A}^{(0)}\).
The MoE is employed here for three key reasons: its gating mechanism enables adaptive routing to specialized experts, its expert specialization enhances expressivity while being inherently suited to sparse, noisy PPI–GO signals (by focusing expert capacity on relevant patterns), and this design achieves these benefits with substantially lower parameter cost than re-training large pre-trained encoders.

The MoE module operates per protein node to integrate multi-modal features. For each protein node $p$ with input features $\mathbf{H}_p$, the MoE comprises $K$ experts (each implemented as a linear layer $f_k:\mathbb{R}^{2d_{\text{seq}}} \rightarrow  \mathbb{R}^{d_h}$) and a gating network $\mathbf{g}_p=\mathrm{Softmax}\bigl(\mathrm{Linear}(\mathbf{H}_{p})\bigr) \in \mathbb{R}^K$.
Then, the fused node embedding is calculated as follows:
\begin{equation}
\tilde{\mathbf{H}}_{p}=g_{\phi}(\mathbf{H}_{p})=\sum_{k=1}^K \mathbf{g}_{p,k} \cdot f_k(\mathbf{H}_{p}) \in \mathbb{R}^{d_h},
\label{eq:H}
\end{equation}
where \(\phi\) denotes learnable parameters of the MoE.

We train the model incorporating this MoE module to generate protein embeddings that intrinsically encode external relational knowledge. During graph generation, it serves as a condition encoder that transforms intrinsic features into context-rich representations, supplies biologically meaningful conditioning signals to guide denoising, and enables embeddings to capture complex relational patterns in the heterogeneous PPI–GO graph.

\subsubsection{Training objective}
In each training step, we first sample a set of central protein nodes from the graph \(\mathcal{G}^{(o)}\), induce an ego-graph for each sampled node, and derive the clean adjacency tensor \(\mathbf{A}^{(0)}\) and the node feature matrix ($\mathbf{H}$ and $\mathbf{Z}$) corresponding to each ego-graph.
We then draw a timestep \(t \sim \mathrm{Uniform}\{1, \dots, T\}\) and generate a noisy adjacency \(\mathbf{A}^{(t)} \sim q(\cdot \mid \mathbf{A}^{(0)})\) via the forward process (Eq.~\ref{eq:forward}).
Concurrently, the OT-aligned intrinsic input \(\mathbf{H}\) is fed into the condition encoder (Eq.~\ref{eq:H}), which produces fused context embeddings \(\tilde{\mathbf{H}}\) to serve as guiding signals for the denoiser.

To make the dependence on the $t$, $\mathbf{A}^{(0)}$ and $\mathbf{H}$ explicit, we define the per-timestep reconstruction loss as follows:

\begin{equation}
\mathcal{L}_t(\theta,\phi)= \mathbb{E}_{\mathbf{A}^{(0)}, \mathbf{H}, \mathbf{A}^{(t)}}\left[\frac{1}{|\mathcal{V}|^2}\sum_{1\leq i,j \leq |\mathcal{V}|}\mathrm{CE}\bigl(\mathbf{A}_{i,j}^{(0)},\widehat{\mathbf{P}}^{(t)}_{i,j} \bigr)\right],
\label{eq:reconstruction}
\end{equation}
where $\widehat{\mathbf{P}}^{(t)}$ is the denoiser's output at timestep $t$ according to Eq.~\ref{eq:condpred}. The generative objective is defined as the average over all timesteps as follows:
\begin{equation}
\label{eq:time_integrated_loss}
\mathcal{L}(\theta,\phi)
\;=\;
\mathbb{E}_{t\sim\mathrm{Uniform}\{1,\dots,T\}}\bigl[\mathcal{L}_t(\theta,\phi)\bigr]
\;=\;
\frac{1}{T}\sum_{t=1}^T \mathcal{L}_t(\theta,\phi).
\end{equation}

By forcing $\widehat{\mathbf{P}}^{(t)}$ to accurately reconstruct each edge’s type from its noisy counterpart, conditioned on both protein and GO‑term embeddings, the condition encoder $g_{\phi}$ learns to extract informative cues from $\mathbf{H}$ for the Graph Transformer, while the latter takes advantage of $\tilde{\mathbf{H}}$ for the reconstruction objective and conversely offer graph-aware knowledge to $\phi$ via propagated gradients.

\noindent\textbf{Scalability.}
To scale this conditional generative modeling, we sample protein-centered ego-graphs using a popular neighbor sampler~\cite{GraphSAGE}. In each step of the training course, we randomly draw a mini-batch of protein nodes from $\mathcal{P}$, expand each into a sampled ego-graph with the hop and type-specific number of neighbors, and perform the training step of conditional diffusion model on these ego-graphs. Graph sampling helps us avoid ego-graphs of hub nodes, which will break the memory limitation.
In this sense, each ego-graph is an instance drawn from the underlying data distribution.
%We pair each ego-graph (as a batch unit) with a single timestep \(t \sim \mathrm{Uniform}\{1,\dots,T\}\)—this random sampling of $t$ across batches eliminates the need for explicit integration over all timesteps, as sufficient iterations ensure coverage of all diffusion stages. Over many training iterations, this Monte Carlo sampling covers all diffusion stages, yielding an unbiased estimate of the full objective while reducing per-batch computation by a factor of $T$.
\begin{algorithm}[t]  
\caption{Training the Conditional Graph Diffusion Model}  
\label{alg:end-to-end-training}  
\begin{algorithmic}[1]
\REQUIRE
Heterogeneous graph $\mathcal{G}$, Intrinsic protein embeddings $\mathbf{H}$, GO term embeddings $\mathbf{Z}$, Diffusion steps $T$, and Transition matrices $\{Q^{(t)}\}_{t=1}^T$.
\ENSURE  
Trained $g_\phi$ and $\Phi_{\theta}$
\STATE Initialize parameters $\theta$ (for $\Phi_{\theta}$) and $\phi$ (for $g_{\phi}$) randomly
\WHILE{not converged}
    \STATE Sample a protein $p$ randomly from the protein set $\mathcal{P}$ 
    \STATE Get $p$-centered ego-graph via \texttt{NeighborLoader}\((\mathcal{G}, \text{center}=p)\)
    \STATE \textit{// Forward Process} 
    \STATE Prepare $\mathbf{A}^{(0)}, \mathbf{H}, \mathbf{Z}$ // Clean adjacency tensor and embeddings of the protein and GO term nodes of the sampled ego-graph
    \STATE Sample $t \sim \mathrm{Uniform}\{1,\dots,T\}$
    \STATE Sample $\mathbf{A}^{(t)}\sim q(\mathbf{A}^{(t)} \mid \mathbf{A}^{(0)})$ // see Eq. \ref{eq:forward}
    % \STATE $\mathbf{A}^{(t)} \gets \mathbf{A}^{(0)}\bar{Q}^{(t)}$ // Generate noisy adjacency via t-step transition probability (Eq. \ref{eq:transition})

    \STATE \textit{// Reverse Denoising}
    \STATE $\tilde{\mathbf{H}} \gets g_{\phi}(\mathbf{H})$ // Encode protein condition
    \STATE $\mathcal{C} \gets \{\tilde{\mathbf{H}}\} \cup \{\mathbf{Z}\}$ // Assemble conditioning signals
    \STATE $\widehat{\mathbf{P}}^{(t)} \gets \Phi_{\theta}(\mathbf{A}^{(t)}, \mathcal{C}, t)$ // Predict clean edge distribution
    
    \STATE \textit{// Optimization}
    \STATE $\mathcal{L}_t(\theta,\phi) \leftarrow \frac{1}{|\mathcal{V}|^2}\sum_{1\leq i,j \leq |\mathcal{V}|}\mathrm{CE}\bigl(\mathbf{A}_{i,j}^{(0)},\widehat{\mathbf{P}}^{(t)}_{i,j}\bigr)$ // see Eq.~\ref{eq:reconstruction}
    \STATE Update parameters $\{\theta,\phi\}\leftarrow\mathrm{AdamW}(\nabla_{\theta,\phi}\mathcal{L}_t(\theta,\phi))$
\ENDWHILE
\end{algorithmic}
\end{algorithm}

\subsection{Protein function prediction}
\label{sec:prediction}

With the condition encoder learned by CGG, we use its produced protein embeddings \(\tilde{\mathbf{H}}_p\) (see Eq.~\ref{eq:H}) to tackle the protein function prediction task.
For each specific ontology $o\in\{\text{MF, BP, CC}\}$, $\mathcal{O}^{(o)}$ are our considered GO terms.
Our prediction task is multi-class multi-label classification, that is, we perform binary classification regarding each GO term, and each protein is allowed to be labeled with multiple GO terms.
%: a protein may map to multiple non-mutually exclusive GO terms.
Formally, for each protein $p$, let \(Y_p = \{Y_{p,\tau}\}_{\tau \in \mathcal{O}^{(o)}}\) denotes a collection of random variables, with \(Y_{p,\tau}=1\) if $p$ has the GO term \(\tau\), otherwise $Y_{p,\tau}=0$. Our goal is to estimate the ground-truth probability distribution $q(Y_{p} | p)$.

To this end, we use a multi-layer perceptron (MLP) as the classifier: it takes \(\tilde{\mathbf{H}}_p\) as input, outputs logits for each GO term \(\tau\), and applies a sigmoid activation to produce \(\hat{p}(Y_{p,\tau}|p)\).
Under the standard conditional independence assumption, the likelihood of ground-truth realization $y_p$ can be factorized into individual GO term's likelihoods.
Thus, we optimize the classifier by maximize the log-likelihood as follows:
%\begin{align*}
$$
\log{\prod_{p\in\mathcal{P}}\hat{p}(Y_p =y_p |p)}=\log{\prod_{p\in\mathcal{P}}\prod_{\tau\in\mathcal{O}^{(o)}}\hat{p}(Y_{p,\tau}=y_{p,\tau} |p)}
=\sum_{p\in\mathcal{P}}\sum_{\tau\in\mathcal{O}^{(o)}}\log{ \hat{p}(Y_{p,\tau}=y_{p,\tau} |p) }.
%\end{align*}
$$
%Downstream training optimizes the classifier (MLP) and condition encoder (MoE) parameters \(\psi, \phi\) to maximize this likelihood, using focal loss to mitigate class imbalance.

At the inference stage, estimated probabilities are computed using this MLP and sigmoid.
However, GO terms are actually non-mutually exclusive.
Therefore, we perform hierarchical true-path propagation as follows:
\begin{equation}
\label{eq:propagate}
\hat{p}(Y_{p,\tau}=1) \leftarrow \max\bigl(\hat{p}(Y_{p,\tau}=1),\; \max_{\tau'\in\mathcal{D}(\tau)} \hat{p}(Y_{p,\tau'}=1)\bigr),
\end{equation}
where \(\mathcal{D}(\tau)\) denotes the set of descendants of \(\tau\).
The implementation details and hyperparameters are provided in Section~\ref{sec:experiment setup}.
\section{Analysis}
\label{sec:analysis}

To better understand \ours, one of the most crucial questions is what kind of knowledge our objective Eq.~\ref{eq:reconstruction} pushes $\tilde{\mathbf{H}}_{p}$ to encode.
In this section, we answer this question by establishing our analysis on the theoretical results of a diffusion representation learning framework Graffe~\cite{Graffe}.
They also employ a conditional diffusion model and optimize a denoising score matching loss regarding both the backbone of denoiser and the condition encoder, which shares the same neural architecture as \ours.
Their theoretical results show that, by this means, the mutual information between the target variable and the condition embedding is lower bounded by their denoising score matching loss, which implies that condition encoder is encouraged to learn representations that are information for reconstructing the clean data.

However, our conditional diffusion model is based on DiGress~\cite{DiGress2022}, where the target variable is discrete, and the diffusion objective is simplified as in cross-entropy form (see Eq.~\ref{eq:reconstruction}) rather than score matching.
In addition, our condition encoder processes the conditional variable (i.e., $\mathbf{H}$ in our case) instead of the target variable ($\mathbf{A}$ in our case) as Graffe.
Therefore, we provide the following proposition for our case.

\begin{proposition}
Minimizing our diffusion objective $\mathcal{L}(\theta, \phi)\triangleq \mathbb{E}_{t}\mathbb{E}_{\mathbf{A}^{(0)}, \mathbf{H}, \mathbf{A}^{(t)} } \text{CE}( \mathbf{A}^{(0)}, \widehat{\mathbf{P}})$ encourages the condition encoder to produce condition embedding that maximizes its conditional mutual information w.r.t. $\mathbf{A}^{(0)}$.
\end{proposition}
\begin{proof}
To keep our notation terse, we will not mention the timestep w.l.o.g.
We first rewrite our objective as follows:
% \begin{align*}
% &\mathcal{L}(\theta,\phi)\\
% =&\sum_{A^{(0)}} q(A^{(0)}) \sum_{\mathbf{H}} q( \mathbf{H} |A^{(0)}) \sum_{A^{(t)}} q(A^{(t)}|A^{(0)})\\
% &\cdot\left[ -\log{ p_{\theta}(A^{(0)}|A^{(t)}, \tilde{\mathbf{H}}) } \right]  \\
% =&\sum_{\tilde{\mathbf{H}}} \sum_{A^{(t)}} 
%     \underbrace{ \bigg( \sum_{A^{(0)}} q(A^{(0)}) q(A^{(t)}|A^{(0)}) \sum_{\mathbf{H}} q(\mathbf{H}|A^{(0)}) \mathbf{1}(\tilde{\mathbf{H}}=g_{\phi}(\mathbf{H})) \bigg) }_{p(\mathbf{A}^{(t)},\tilde{\mathbf{H}})} \\
% &\quad \cdot\bigg( \sum_{A^{(0)}} \frac{ q(A^{(0)}) q(A^{(t)}|A^{(0)}) \sum_{\mathbf{H}} q(\mathbf{H}|A^{(0)}) \mathbf{1}(\tilde{\mathbf{H}}=g_{\phi}(\mathbf{H})) }{ p(A^{(t)}, \tilde{\mathbf{H}}) } \\
% &\quad\quad \cdot \left[ -\log p_{\theta}(A^{(0)}|A^{(t)},\tilde{\mathbf{H}}) \right] \bigg) \\
% =&\sum_{A^{(t)}} \sum_{\tilde{\mathbf{H}}} p(A^{(t)}, \tilde{\mathbf{H}}) \, \mathbb{E}_{A^{(0)} \sim p^{*}} \left[ -\log{ p_{\theta}(A^{(0)} | A^{(t)}, \tilde{\mathbf{H}}) } \right],
% \end{align*}

\begin{align*}
&\mathcal{L}(\theta,\phi)=\sum_{\mathbf{A}^{(0)}} q(\mathbf{A}^{(0)}) \sum_{\mathbf{H}} q( \mathbf{H} |\mathbf{A}^{(0)}) \sum_{\mathbf{A}^{(t)}} q(\mathbf{A}^{(t)}|\mathbf{A}^{(0)}) \cdot\left[ -\log{ p_{\theta}(\mathbf{A}^{(0)}|\mathbf{A}^{(t)}, \tilde{\mathbf{H}}) } \right]  \\
=&\sum_{\tilde{\mathbf{H}}} \sum_{\mathbf{A}^{(t)}} 
    \underbrace{ \bigg( \sum_{\mathbf{A}^{(0)}} q(\mathbf{A}^{(0)}) q(\mathbf{A}^{(t)}|\mathbf{A}^{(0)}) \sum_{\mathbf{H}} q(\mathbf{H}|\mathbf{A}^{(0)}) \mathbf{1}_{\tilde{\mathbf{H}}=g_{\phi}(\mathbf{H})} \bigg) }_{p(\mathbf{A}^{(t)},\tilde{\mathbf{H}})} \\
&\quad \cdot\bigg( \sum_{\mathbf{A}^{(0)}} \frac{ q(\mathbf{A}^{(0)}) q(\mathbf{A}^{(t)}|\mathbf{A}^{(0)}) \sum_{\mathbf{H}} q(\mathbf{H}|\mathbf{A}^{(0)}) \mathbf{1}_{\tilde{\mathbf{H}}=g_{\phi}(\mathbf{H})} }{ p(\mathbf{A}^{(t)}, \tilde{\mathbf{H}}) }  \cdot \left[ -\log p_{\theta}(\mathbf{A}^{(0)}|\mathbf{A}^{(t)},\tilde{\mathbf{H}}) \right] \bigg) \\
=&\sum_{\mathbf{A}^{(t)}} \sum_{\tilde{\mathbf{H}}} p(\mathbf{A}^{(t)}, \tilde{\mathbf{H}}) \, \mathbb{E}_{\mathbf{A}^{(0)} \sim p^{*}} \left[ -\log{ p_{\theta}(\mathbf{A}^{(0)} | \mathbf{A}^{(t)}, \tilde{\mathbf{H}}) } \right],
\end{align*}
where $\mathbf{1}_{\cdot}$ denotes the indicator function, and $p^{*}$ denotes the posterior of the clean data $\mathbf{A}^{(0)}$ given the noisy sample $\mathbf{A}^{(t)}$ and condition embedding $\tilde{\mathbf{H}}$, namely $p(\mathbf{A}^{(0)} | \mathbf{A}^{(t)}, \tilde{\mathbf{H}})$.

Then, by the definition of KL-divergence, we can further rewrite it as follows:
\[
\mathcal{L}(\theta,\phi) = H(\mathbf{A}^{(0)} | \mathbf{A}^{(t)}, \tilde{\mathbf{H}} ) + \text{KL}\big( p{}^{*} \| p_{\theta}(\mathbf{A}^{(0)} | \mathbf{A}^{(t)}, \tilde{\mathbf{H}} ) \big)\geq H(\mathbf{A}^{(0)} | \mathbf{A}^{(t)}, \tilde{\mathbf{H}} ),
\]
where the last inequality comes from the fact that the KL-divergence is non-negative.

Meanwhile, recall the equality:
%\begin{align*}
$I(\mathbf{A}^{(0)} ; \tilde{\mathbf{H}} | \mathbf{A}^{(t)}) = H( \mathbf{A}^{(0)} | \mathbf{A}^{(t)}) - H( \mathbf{A}^{(0)} | \tilde{\mathbf{H}}, \mathbf{A}^{(t)})$.
%\end{align*}
Combining this equality with the above inequality, we have:
\begin{align*}
I(\mathbf{A}^{(0)} ; \tilde{\mathbf{H}} | \mathbf{A}^{(t)}) &\geq H( \mathbf{A}^{(0)} | \mathbf{A}^{(t)}) - \mathcal{L}(\theta, \phi),
\end{align*}
where the first term in the RHS is independent of both $\phi$ and $\theta$.
Hence, this inequality indicates that minimizing our diffusion objective is to maximize a lower bound of the LHS, which completes the proof.
\end{proof}
As the adjacency $\mathbf{A}$ contains extrinsic information about proteins, maximizing the conditional mutual information means that the condition embeddings $\tilde{\mathbf{H}}$ absorb extrinsic information.
Since the condition encoder $g_{\phi}(\cdot)$ takes intrinsic information $\mathbf{H}$ to produce $\tilde{\mathbf{H}}$, it fuses both kinds of information.
Besides, the noisy samples at each timestep preserve the information of $\mathbf{A}^{(0)}$ at different levels and scales, which, intuitively speaking, tends to make $\tilde{\mathbf{H}}$ robust and versatile.
\section{Experiments}
\label{sec:exp}

\begin{table}[t]
\centering
\begin{tabular}{ccccccc}
\hline
\multirow{2}{*}{Methods} &
  \multicolumn{2}{c}{MF} &
  \multicolumn{2}{c}{BP} &
  \multicolumn{2}{c}{CC} \\
            & Fmax  & AUPR  & Fmax           & AUPR  & Fmax           & AUPR  \\ \hline
Diamond     & 0.592 & 0.387 & 0.429          & 0.197 & 0.573          & 0.283 \\
BlastKNN    & 0.616 & 0.484 & 0.445          & 0.258 & 0.596          & 0.384 \\
TALE        & 0.260 & 0.158 & 0.253          & 0.152 & 0.548          & 0.510 \\
ATGO        & 0.454 & 0.442 & 0.396          & 0.341 & 0.602          & 0.596 \\
DeepGOCNN   & 0.396 & 0.326 & 0.323          & 0.254 & 0.573          & 0.567 \\
DeepGOPlus  & 0.589 & 0.548 & 0.438          & 0.365 & 0.626          & 0.618 \\
ATGO+       & 0.622 & 0.599 & 0.456          & 0.399 & 0.633          & 0.636 \\ \hline
GearNet     & 0.612 & 0.627 & 0.437          & 0.402 & 0.622          & 0.655 \\
ESM-GearNet & 0.635 & 0.658 & 0.451          & 0.421 & 0.637          & 0.671 \\ \hline
TALE+       & 0.602 & 0.543 & 0.427          & 0.327 & 0.608          & 0.591 \\
DeepGO      & 0.301 & 0.204 & 0.328          & 0.260 & 0.574          & 0.580 \\
DeepGraphGO & 0.562 & 0.533 & 0.432          & 0.389 & 0.634          & 0.590 \\
DPFunc      & 0.635 & 0.658 & \textbf{0.466} & 0.434 & \textbf{0.657} & 0.695 \\
\ours &
  \textbf{0.642*} &
  \textbf{0.671*} &
  0.462 &
  \textbf{0.438*} &
  0.653 &
  \textbf{0.697} \\ \hline
\end{tabular}
\caption{Comparison of baseline methods by AUPR and Fmax on various GO branches. Results are reported as averages over 5 independent runs. 
Statistical significance was assessed via Welch's independent samples t-test (equal variances not assumed) with two-tailed p-values, where values marked with an asterisk (*) indicate the baseline method performed significantly better than DPFunc ($p<0.05$) and all other values show no significant improvement.}
\label{tab:main-result}
\end{table}
To assess the effectiveness of \ours in protein function prediction tasks, this section presents a comprehensive experimental evaluation, with content organized to align with subsequent detailed analyses. 
Specifically, we first introduce the \textit{Experiment Setup}, which covers evaluation metrics and implementation specifics—laying a solid foundation for the reproducibility of our experiments.
Next, we present the \textit{Main Results}, where comparisons between \ours and relevant baseline models verify its performance advantages in GO prediction.
We then conduct an \textit{Ablation Study} to dissect the contribution of each core component in the \ours, clarifying the rationale behind our design choices.
Following this, a \textit{Hyperparameter Sensitivity} analysis is performed to examine how key hyperparameters influence model performance, ensuring the robustness of our framework.
Finally, we complement quantitative results with \textit{Qualitative Evaluation and Case Study}, which provide intuitive insights into \ours’ predictive behavior and further deepen the understanding of its effectiveness in capturing biological meaningful patterns.
\subsection{Experiment Setup}
\label{sec:experiment setup}
\noindent\textbf{Baselines.}
To ensure a comprehensive evaluation of our proposed method, we adopt the same set of baseline models and their corresponding evaluation results as utilized in the DPFunc study \cite{DPFunc_NatCommun2024}. This approach facilitates a consistent and fair comparison across various methodologies. The baseline models encompass:
\textit{(1) Sequence Alignment-Based Methods:} Traditional tools that rely on sequence similarity for function prediction, including BlastKNN~\cite{blastknn_cafa} and DIAMOND~\cite{diamond}.
\textit{(2) Sequence-Based Deep Learning Methods:} Models that leverage deep learning architectures to extract features from protein sequences, such as DeepGOCNN~\cite{kulmanov_deepgozero_2022}, TALE~\cite{hou2021tale}, and ATGO~\cite{zhang2022atgo}.
\textit{(3) PPI Network-Based Methods:} Approaches that incorporate protein-protein interaction networks into their predictive frameworks, exemplified by DeepGO~\cite{DeepGO_2017} and DeepGraphGO~\cite{DEEPGRAPHGO}.
\textit{(4) Composite Methods:} Hybrid models that integrate sequence-based predictions with alignment-based methods to enhance performance, including DeepGOPlus~\cite{deepgoplus_Kulmanov2019}, TALE+~\cite{hou2021tale}, and ATGO+~\cite{zhang2022atgo}.

Additionally, we include GearNet~\cite{zhang2022gearnet} and ESM-GearNet~\cite{zhang2023systematic} as structural baselines in our main experiments.
For these models, we extract frozen embeddings from their pre-trained encoders and train the same  classifier as used in our pipeline for protein function prediction tasks, ensuring a fully consistent experimental setting across all methods.

By employing these established baselines and their reported results, we aim to provide a reliable assessment of our method's effectiveness in protein function prediction tasks.

\noindent\textbf{Metrics.}
To ensure direct and fair comparability with DPFunc~\cite{DPFunc_NatCommun2024}, which serves as the benchmark work targeted in this study, our evaluation aligns with that work in terms of core metrics and computational logic, though notational conventions are independently defined herein.
Specifically, we use two standard metrics, Fmax and AUPR, to assess model performance: Fmax refers to the maximum harmonic mean of precision and recall across all confidence thresholds, and AUPR is computed by integrating the precision–recall curve over the full range of confidence thresholds.

Consistent with this alignment, the step-by-step computation of Fmax and AUPR aligns logically with that in DPFunc.
Additionally, all predicted probabilities \(\hat{p}_{p,\tau}\) undergo hierarchical true-path propagation as post-processing, a procedure that is also logically equivalent to the one detailed in Eq.~\ref{eq:propagate}.
Together, these consistencies guarantee the fairness of performance comparisons between our model and DPFunc.

\noindent\textbf{Implementation Details.}
For protein sequences we employed the ESM-1b~\cite{ESM_Rives2021} model in inference mode to generate residue-level representations; a mean-pool over all tokens yielded a single 1\,280-D vector per protein.
Structural embeddings were encoded with the officially released GearNet-Edge~\cite{zhang2022gearnet} checkpoint; the default sum readout was applied to produce a 3\,072-D structure embedding.
To align structural embeddings with the sequence embedding space, we project the former into the latter using the Sink-horn algorithm for approximation, with a regularization parameter \(\varepsilon = 1e-3\). Iterations terminate when the change in transport cost falls below \(1e-6\).
Both sequence and structure vectors spaces were frozen throughout both pre-training and downstream fine-tuning.

To train hyperbolic embeddings for Gene Ontology (GO) terms (capturing their hierarchical relationships), we use a standard hyperbolic embedding framework with consistent configurations across all GO branches. The Poincaré manifold~\cite{poincare} is adopted to model GO’s tree-like hierarchy, with a distance-based energy function optimizing the distinction between true hierarchical edges (using \texttt{is\_a} and \texttt{part\_of} relations) and negative samples.
The model is evaluated via reconstruction accuracy to assess its ability to recover original GO edges.
These vectors were also frozen during the pre-training of Conditional Graph Generation (CGG)-based information fusion.

CGG-based information fusion utilized four NVIDIA H100 GPUs with AdamW optimization~\cite{adamw}. Learning rates were adapted to each ontology branch’s scale: 2e-4 for BP and MF, and 5e-5 for CC (smallest sample size); weight decay was uniformly set to 1e-12 across all branches. Pre-training took approximately 2–3 hours per branch.

For downstream fine-tuning (GO term classification), a single NVIDIA A100 GPU was used with AdamW and a OneCycleLR scheduler.
The scheduler first linearly increased the learning rate to branch-specific peak values (8e-4 for BP, 2e-4 for MF, and 7e-4 for CC) and then gradually decreased it, allowing rapid exploration of the parameter space in early stages while refining with smaller updates later to preserve the distributions of the condition encoder.
This approach balanced convergence speed and stability when updating both the classifier (large parameter size) and the condition encoder under limited training epochs and adjusted batch sizes per branch.
Weight decay was set to 1e-4 (distinct from CGG’s 1e-12) to provide stronger regularization for  parameter space in downstream classification.
Fine-tuning required ~10–20 minutes per branch.

\subsection{Main Results}
\label{subsec:mainexp}
\begin{figure}[t] 
    \centering    \includegraphics[width=0.8\linewidth]{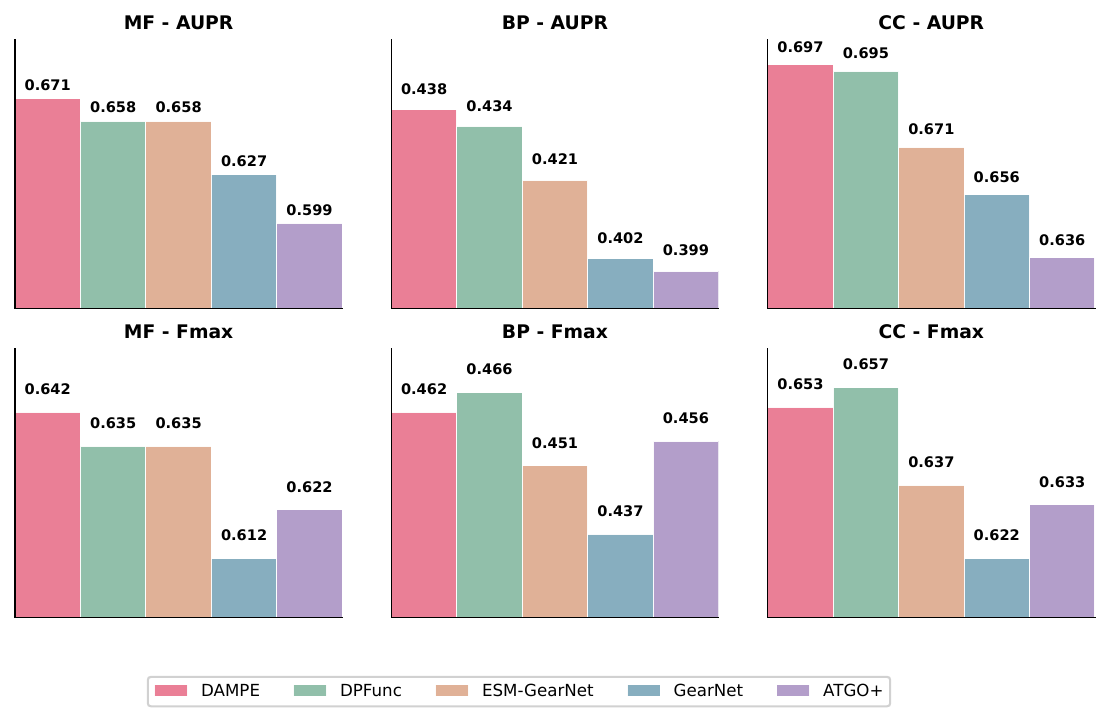}
    \caption{Performance comparison of five models (DAMPE, DPFunc, ESM-GearNet, GearNet, ATGO+) across Gene Ontology (GO) ontology using AUPR and Fmax metrics. DPFunc is the state-of-the-art baseline with domain-guided structure information; ESM-GearNet is a multi-modal baseline fusing intrinsic information; GearNet is a structure-only baseline with a residue-level geometric encoder; ATGO+ is a sequence-only baseline with alignment-enhanced features.}
    \label{fig:main_result}
\end{figure}
To comprehensively assess the performance of protein function prediction methods, we conduct evaluations across three Gene Ontology (GO) ontology (MF, BP and CC) using two benchmark metrics: Fmax and AUPR.
All results are reported as the average of 5 independent experimental runs to ensure statistical stability, with bold values in Table \ref{tab:main-result} indicating the best performance for each metric,
and those marked with (*) denote results significantly superior to the state-of-the-art (SOTA) method DPFunc (\(p<0.05\)).
For intuitive comparison of representative methods, we further visualize the performance of five key models in Figure \ref{fig:main_result}, focusing on the core trends that underscore the value of multi-modal information integration. 

We begin our analysis by examining how input modalities shape predictive performance, starting with sequence-only baselines.
Among these methods, ATGO+ achieves the strongest results—for example, its AUPR scores reach 0.599 (MF) and 0.636 (CC), which confirms that advanced sequence encoding provides a solid foundation for protein function prediction.
However, sequence alone has clear limitations: even ATGO+ is outperformed by methods that incorporate structural information, as seen with GearNet.
By leveraging 3D structural cues alongside sequence, GearNet improves CC-AUPR by 0.02 pp compared to ATGO+, highlighting that structural data is a critical complement to sequence—particularly for tasks like CC, where the correlation between protein structure and localization is inherently strong.
This benefit of multi-modal fusion is further reinforced by ESM-GearNet: by fusing ESM-2 sequence embeddings with GearNet structural features, it delivers additional gains (e.g., MF-AUPR=0.658 vs. GearNet’s 0.627), showing that dedicated integration of sequence and structure can unlock performance beyond what single-modal inputs achieve.
We then turn to methods that incorporate extrinsic information (e.g., PPI, InterPro) to explore how relational or domain knowledge enhances performance.
These methods generally outperform sequence-only and sequence+structure baselines in at least one branch—for instance, DeepGraphGO, which relies on PPI and InterPro, outperforms several sequence-only methods in CC—confirming that extrinsic data addresses gaps left by intrinsic (sequence/structure) inputs alone.
The current SOTA, DPFunc, builds on this by combining sequence, structure, and InterPro to lead in BP-Fmax (0.466) and CC-Fmax (0.657), further demonstrating the value of integrating domain information and revealing the untapped potential of fusing such extrinsic data with intrinsic modalities.

Against this backdrop, our proposed \ours distinguishes itself by holistically integrating intrinsic (sequence/structure) and extrinsic (PPI-GO) information.
This design enables \ours to outperform both baselines and the SOTA in critical metrics: it achieves the highest MF-Fmax (0.642) and MF-AUPR (0.671), as well as the highest BP-AUPR (0.438)—all significantly better than DPFunc (\(p<0.05\)).
Even in CC, where structural cues are most impactful, \ours matches DPFunc’s AUPR (0.697 vs. 0.695) and is only marginally lower in Fmax (0.653 vs. 0.657)—a near-parity that underscores its ability to compete in structure-sensitive tasks without relying on InterPro.
Collectively, these results (supported by the visualized trends in Figure \ref{fig:main_result}) confirm that \ours’ balanced integration of intrinsic and extrinsic information addresses the limitations of prior methods, delivering consistent improvements in challenging branches while maintaining strong competitiveness across all evaluated tasks.

Overall, \ours can both surpass the SOTA DPFunc on several metrics (notably MF and BP AUPR) and match its performance on CC, demonstrating the method's ability to deliver superior or comparable SOTA-level results.

\subsection{Ablation Study}
\label{subsec:ablation}
To validate the contributions of each component in our framework, we conducted a series of ablation experiments.

% \subsubsection{Alignment Ablations}
% \label{subsec:aliment}
\subsubsection{Evaluating Alignment Strategies for Multimodal Fusion} 
\begin{table}[htbp]
\centering
\begin{tabular}{ccccccc}
\toprule
\multirow{2}{*}{Alignment} & \multicolumn{2}{c}{MF}          & \multicolumn{2}{c}{BP}          & \multicolumn{2}{c}{CC}          \\
                     & Fmax  & AUPR  & Fmax  & AUPR  & Fmax  & AUPR  \\ \hline
None                 & 0.618 & 0.637 & 0.445 & 0.415 & 0.629 & 0.663 \\
Contrastive Learning & 0.612 & 0.603 & 0.417 & 0.377 & 0.617 & 0.647 \\
Cross-Attention      & 0.585 & 0.587 & 0.347 & 0.289 & 0.600 & 0.632 \\
ESM-GearNet          & 0.635 & 0.658 & 0.451 & 0.421 & 0.637 & 0.671 \\
OT-based     & \textbf{0.640} & \textbf{0.667} & \textbf{0.458} & \textbf{0.434} & \textbf{0.650} & \textbf{0.690} \\ 
\bottomrule
\end{tabular}
\caption{Comparison of intrinsic data alignment strategies for GO prediction (MF, BP, CC). ``None'' = direct concatenation of frozen sequence and structure embeddings (no alignment). ``OT-based Representation Alignment'' = Optimal Transport (OT)-based representation alignment, with MoE fusion trained from scratch. Bold entries indicate the best score per metric. Results are means over five runs.}
\label{tab:Aligment}
\end{table}
To systematically assess the necessity and efficacy of explicit alignment mechanisms, we conducted an ablation study to compare four fusion paradigms:
\textit{(1) Na\"ive Concatenation}: Directly concatenates frozen sequence and structural embeddings without any alignment, serving as a minimal fusion baseline.
\textit{(2) Contrastive Alignment}~\cite{CLIP}: Performs unsupervised projection into a shared latent space via dual MLP heads trained with contrastive loss—where positive pairs (sequence and structure embeddings of the same protein) are attracted and negative pairs are repelled. The pre-trained projection outputs are concatenated to support downstream prediction.
\textit{(3) ESM-GearNet Serial Fusion}~\cite{zhang2023systematic}: A serial fusion approach that uses the output of the pre-trained Protein Language Model (ESM-2) as the input to the GearNet structure encoder; we adopt the output of their open-source model as protein representations.
\textit{(4) OT-based representation alignment}: Employs explicit domain adaptation to align structural embeddings with the sequence embedding space via OT, leveraging the superior information density of representations. The transformed structural features are concatenated with the original sequence features.

Quantitative analysis (Table~\ref{tab:Aligment}) demonstrates that explicit alignment is critical for multimodal fusion, with serial fusion (ESM-GearNet) and Optimal Transport (OT)-based representation alignment achieving 0.006-0.022 pp increase in Fmax over naive concatenation—confirming significant modality gap reduction. OT emerges as the superior strategy, attaining peak performance (CC Fmax=0.650 and AUPR=0.690) while avoiding serial fusion's costly co-training overhead. Notably, contrastive alignment underperformed the baseline, suggesting task-agnostic alignment may disrupt functional semantics. These results validate OT-based representation alignment as an efficient and effective paradigm for cross-modal integration.

\subsubsection{Extrinsic Knowledge Integration Analysis}

\begin{table*}[t]
\centering
\begin{tabular}{ccccccc}
\toprule
\multirow{2}{*}{Methods}                              & \multicolumn{2}{c}{MF} & \multicolumn{2}{c}{BP} & \multicolumn{2}{c}{CC} \\
                                                      & Fmax       & AUPR      & Fmax       & AUPR      & Fmax       & AUPR      \\ \hline
                                                      
GAT                                                  & 0.546      & 0.523     & 0.391      & 0.337     & 0.642      & 0.691     \\
SAGE                                                  & 0.508      & 0.479     & 0.313      & 0.244     & 0.589      & 0.610     \\

\ours w/o CGG & 0.640      & 0.667     & 0.458      & 0.434     & 0.650      & 0.690     \\
\ours & \textbf{0.642} & \textbf{0.671} & \textbf{0.462} & \textbf{0.438} & \textbf{0.653} & \textbf{0.697} \\ \bottomrule
\end{tabular}
\caption{Comparison of PPI integration methods on GO prediction. 
GAT and SAGE are standard GNN architectures applied directly to the PPI network to extract node features for downstream classification. 
Boldface denotes the best score per metric. Results are means over five runs.}
\label{tab:PPI}
\end{table*}
We systematically evaluated three extrinsic knowledge integration strategies: (1) Mixture-of-Experts (MoE) fusion with random initialization, (2) MoE initialized via conditional graph generation (CGG) information fusion, and (3) conventional GNN-based PPI integration (GAT~\cite{DeepGS}/SAGE~\cite{GraphSAGE}).

Quantitative results (Table\ref{tab:PPI}) demonstrate that the pre-trained MoE significantly outperforms both randomly initialized MoE (Fmax 0.002-0.004 pp\textuparrow across ontology) and traditional GNN methods, with our full framework achieving the best performance. 
The consistent superiority of CGG—particularly its 0.011-0.149 pp increase in Fmax over GNN baselines—validates our generative approach as the optimal paradigm for extrinsic knowledge fusion.

\begin{table}[]
\centering
\resizebox{\linewidth}{!}{%
\begin{tabular}{lcccccc}
\toprule
\multirow{2}{*}{\textbf{Model}} & \multicolumn{2}{c}{\textbf{Average Inference Time (ms)}} & \multicolumn{2}{c}{\textbf{P90 Latency (ms)}} & \multicolumn{2}{c}{\textbf{Throughput (samples/sec)}} \\
\cmidrule(lr){2-3} \cmidrule(lr){4-5} \cmidrule(lr){6-7}
& Mean & Std. & Mean & Std. & Mean & Std. \\
\midrule
GAT & 5.90 & 1.49 & 6.75 & 3.15 & 10844.89 & 1164.50 \\
\ours & 1.28 & 0.97 & 0.99 & 0.19 & 55160.48 & 8742.65 \\
\midrule
\multicolumn{7}{l}{\small{\textit{Note: Results are averaged over 100 repeated experiments on the BP test set, with outliers filtered using the 3$\sigma$ rule.}}} \\
\multicolumn{7}{l}{\small{\textit{Valid samples after filtering: GAT (98/99/96 for time/P90/throughput); \ours (97/98/95 for time/P90/throughput).}}} \\
\bottomrule
\end{tabular}%
}
\vspace{1mm}
\caption{Comparison of inference cost on GO prediction.}
\label{tab:inference_performance}
\end{table}
To validate \ours’ practicality for protein function prediction, we compare its inference performance with the representative GNN model GAT, which better reflects the efficiency bottlenecks of traditional GNNs due to its attention mechanism, via 100 repeated experiments (3\(\sigma\) outlier filtering) under identical environments (Table \ref{tab:inference_performance}).
Results confirm \ours delivers significant efficiency advantages over GAT: it achieves much lower average inference time and P90 latency (with better stability), alongside substantially higher throughput—critical for both real-time high-throughput protein annotation and large-scale processing of uncharacterized proteins.
These gains stem from \ours’ design, which abandons traditional GNN message-passing and thus reduces redundant computations, while retaining discriminative features critical for function classification—ultimately delivering superior inference efficiency.
% What conclusion can be drawn from expected results?

\subsection{Hyperparameter Sensitivity}
\begin{figure}[htbp]
    \centering
    \begin{subfigure}[b]{0.48\textwidth}  
        \centering
        \includegraphics[width=\textwidth]{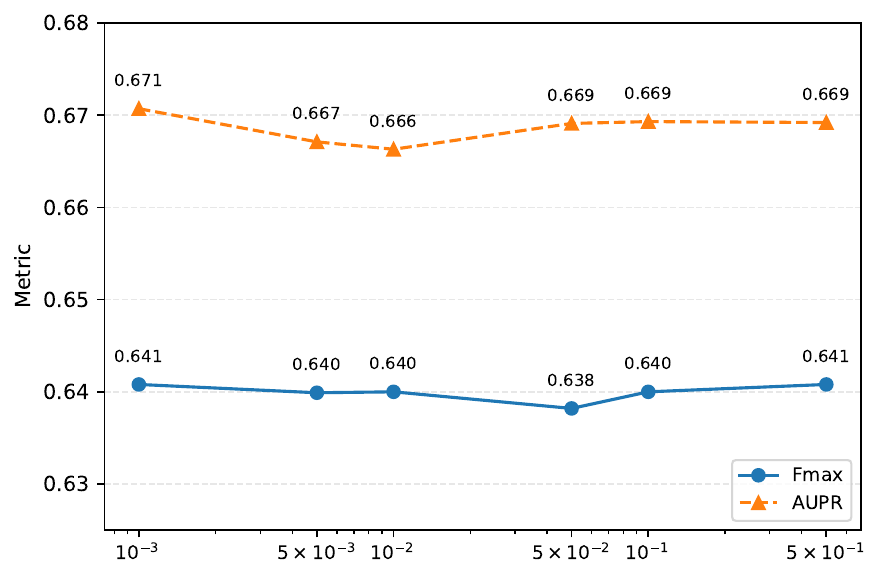}
        \caption{Sensitivity Analysis of $\varepsilon$. The plot evaluates AUPR and Fmax under different epsilon values, demonstrating the model's stable performance across the tested parameter range.}
        \label{fig:epsilon_sensitivity}
    \end{subfigure}
    \hfill  
    \begin{subfigure}[b]{0.48\textwidth}  
        \centering
        \includegraphics[width=\textwidth]{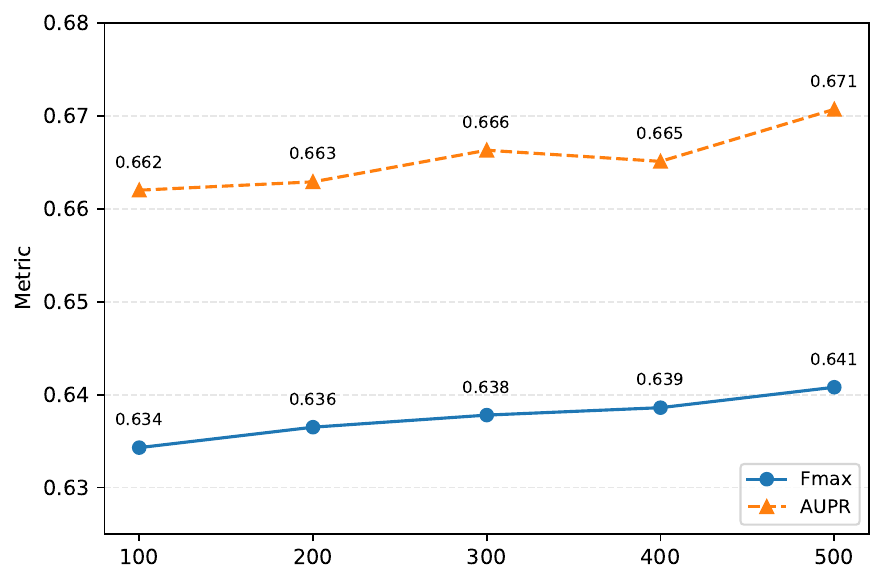}
        \caption{Sensitivity Analysis of Diffusion Steps. The plot evaluates AUPR and Fmax under different diffusion steps, demonstrating the model's stable performance across the tested parameter range.}
        \label{fig:diffusion_sensitivity}
    \end{subfigure}
    \caption{Hyperparameter Sensitivity}
    \label{fig:main}
\end{figure}
To verify the robustness of \ours, we perform a systematic sensitivity analysis on the two core modules that bridge modalities: (i) the entropy-regularisation coefficient $\varepsilon$ of the optimal transport (OT) and (ii) the diffusion-step count $T$ in the Conditional Graph Generation (CGG)-based information fusion.

\subsubsection{Robustness to the weight of entropy-regularisation coefficient $\varepsilon$} 
\label{subsec:sensitivity}
We investigate the effect of the entropy regularization coefficient $\epsilon$ in the Sinkhorn-based optimal transport (OT) problem. 
Figure~\ref{fig:epsilon_sensitivity} summarizes the performance in terms of AUPR and Fmax across a wide range of $\epsilon$ values. 
Two consistent trends can be observed.
Although OT improves downstream performance relative to no alignment (see section~\ref{subsec:ablation}), we observe that model performance is notably robust to the choice of the regularization coefficient \(\varepsilon\) within the tested range. 
We attribute this robustness to (i) the high dimensionality of pre-trained embeddings (1280, 3072), which concentrates pairwise costs and reduces the effective sensitivity of the Sinkhorn kernel to moderate changes in \(\varepsilon\); (ii) the barycentric projection and the downstream classifier, which smooth and compensate for small differences in the transport plan; and (iii) practical cost-matrix scaling and numerical stabilizations in the Sinkhorn implementation~\cite{Cuturi2013_Sinkhorn}. 
Consequently, although very small $\varepsilon$ values asymptotically approach the unregularized OT solution and can yield slight improvements, in our experiments tuning $\varepsilon$ produces only marginal gains: performance remains stable across the tested range and larger $\varepsilon$ values give negligible variation. Hence, selecting a moderate $\varepsilon$ (validated on a held-out split) is sufficient in practice and simplifies deployment.

\subsubsection{Sensitivity to diffusion steps $T$}  
As shown in Figure~\ref{fig:diffusion_sensitivity}, Fmax increases monotonically with the number of diffusion steps over the tested range, while AUPR exhibits mild, non-monotonic fluctuations and attains its maximum at the largest tested value (500 steps).
The observed changes across the sweep are small in magnitude.
The magnitude of the gains is modest, with improvements being incremental rather than dramatic, and this suggests the discrete-diffusion pretraining is robust to the exact choice of steps within the tested range.

\subsection{Qualitative Evaluation and Case Study}
\label{visulize}

\begin{figure}[thbp] 
  \centering
  \includegraphics[width=\textwidth]{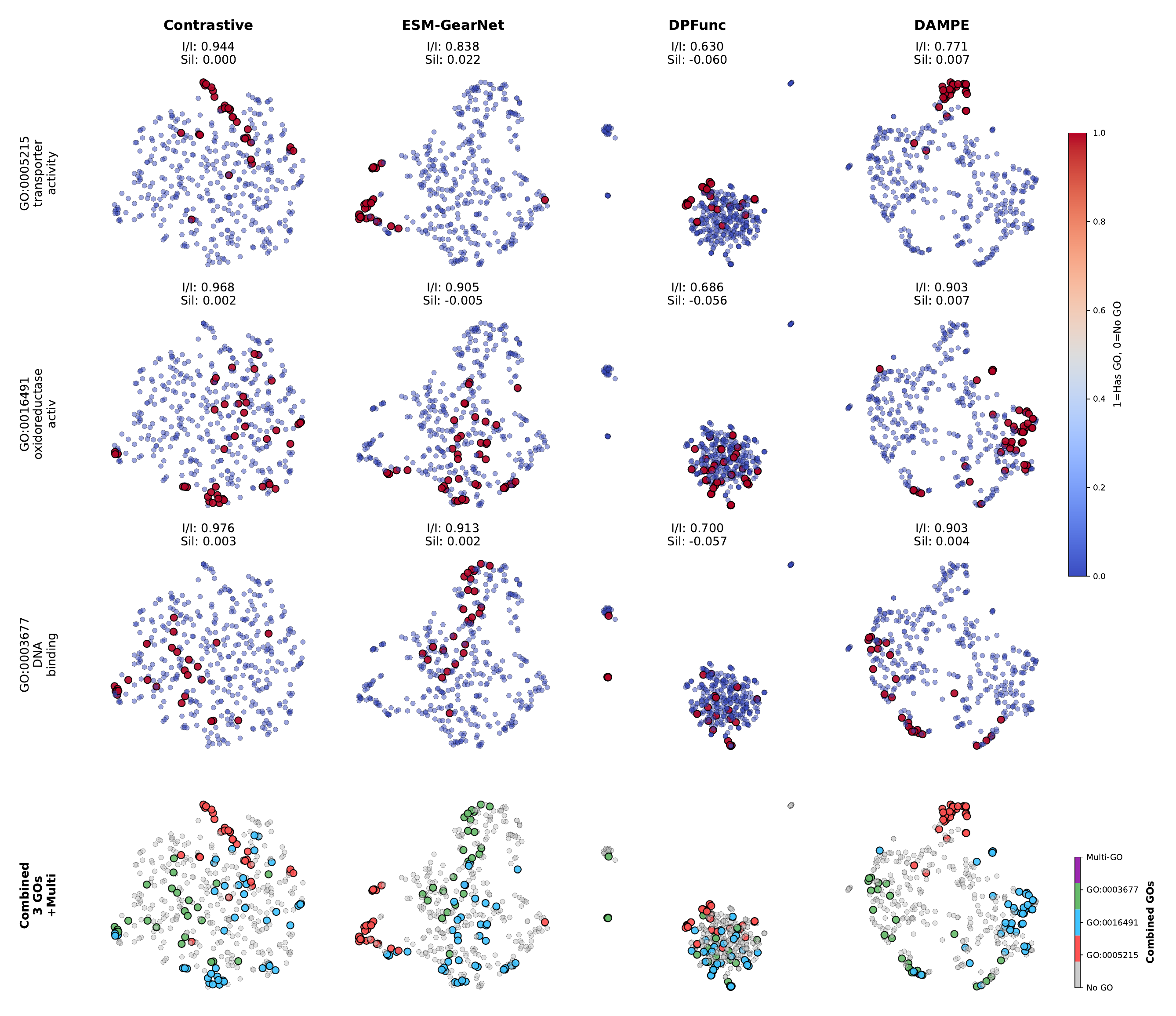}
  \caption{Qualitative evaluation of embeddings from four models (Contrastive, ESM-GearNet, DPFunc, \ours) via UMAP 2D projections and clustering metrics. The figure includes 4 rows: 3 rows for individual Molecular Function GO terms and 1 row for the combined three GO terms. Each subplot shows UMAP-projected embeddings (distinguishing positive/negative samples for individual GO terms, or GO membership for combined terms) and reports two metrics for individual GO terms: intra/inter-cluster distance ratio (measuring cluster compactness) and silhouette score (measuring positive-negative separation). Color bars indicate sample labels for individual and combined GO term projections. This visualization is used to assess the quality of model embeddings in capturing functional patterns of GO terms.}
  \label{fig:case-study}
\end{figure}

We perform a qualitative evaluation of protein embeddings by using dimensionality reduction and clustering to explore how well each method captures protein functional patterns.
For this analysis, we selected three representative Molecular Function (MF) GO terms: transporter activity (GO:0005215), oxidoreductase activity (GO:0016491), and DNA binding (GO:0003677). These terms were chosen for two key reasons: they are relatively close to MF’s ancestral nodes (covering broad, functionally distinct categories) while belonging to separate sub-branches of the MF hierarchy, and they have balanced sample sizes to ensure reliable clustering diagnostics. As shown in Figure~\ref{fig:case-study}, the evaluation combines UMAP 2D projections (to visualize embedding distribution) with two clustering metrics: intra/inter-cluster distance ratio (for assessing cluster compactness) and silhouette score (for measuring positive-negative separation).  

Our proposed \ours shows favorable embedding quality across these three terms: it tends to achieve relatively high intra/inter ratios (transporter: 0.771; oxidoreductase: 0.903; DNA binding: 0.903) and small but positive silhouette scores (transporter: 0.007; oxidoreductase: 0.007; DNA binding: 0.004), suggesting its positive samples may form more compact clusters with clearer separation from negative samples. 
By comparison, the state-of-the-art DPFunc generally yields lower intra/inter ratios (transporter: 0.630; oxidoreductase: 0.686; DNA binding: 0.700) and negative silhouette scores (transporter: -0.060; oxidoreductase: -0.056; DNA binding: -0.057), which could indicate more dispersed positive clusters.
The Contrastive (which generates embeddings via contrastive learning) and ESM-GearNet baselines exhibit marginal differences from each other: while they sometimes match or slightly exceed DPFunc in intra/inter ratios, they rarely achieve consistently positive silhouette scores, implying their separation advantage is limited.  

These visual and numeric observations collectively suggest that \ours' embeddings may be more compact and reliably separable for the selected MF terms, even across broad, distinct sub-branches of the MF hierarchy, than those of the compared methods.
This qualitative analysis supports the idea that \ours may better capture branch-specific intrinsic and extrinsic information, while maintaining generalization across diverse functional contexts within MF.

\section{Discussion}
\label{sec:discussion}
\ours addresses core challenges in multi-modal protein function prediction by unifying intrinsic (sequence/structure) and extrinsic (PPI/GO) information through two tailored mechanisms.
Its Optimal Transport (OT)-based representation alignment alleviate cross-modal heterogeneity. Complementarily, conditional graph generation (CGG)-based information fusion replaces traditional GNNs for extrinsic fusion, mitigating noise of PPI-GO graph and the unrealistic edge independence assumption of GNNs.
Together, these designs enable \ours to advance protein function prediction, outperforming baselines in key tasks while retaining pre-trained models' knowledge and reducing unnecessary computational overhead.  

\ours has inherent limitations tied to two core design choices.
First, its reliance on protein-level embeddings for both protein sequence (e.g., mean-pooled PLM outputs) and structure (e.g., mean-readout of geometric encoders) results in the loss of residue-level details, making it unsuitable for residue-specific tasks.
Second, regarding its CGG stage: compared to traditional GNNs that directly perform message passing on PPI-GO graphs, CGG's generative paradigm for learning extrinsic relational knowledge requires greater computational resources and longer training time.

\section*{Data availability}
All relevant source codes of \ours are available at \url{https://anonymous.4open.science/r/DAMPE-ACD8}. The benchmark datasets used in this study refer to the resource hosted at \url{https://github.com/CSUBioGroup/DPFunc}.

\bibliographystyle{elsarticle-num} 
\bibliography{refs}

\begin{thebibliography}{10}
\expandafter\ifx\csname url\endcsname\relax
  \def\url#1{\texttt{#1}}\fi
\expandafter\ifx\csname urlprefix\endcsname\relax\def\urlprefix{URL }\fi
\expandafter\ifx\csname href\endcsname\relax
  \def\href#1#2{#2} \def\path#1{#1}\fi

\bibitem{Hirokawa2009}
N.~Hirokawa, Y.~Noda, Y.~Tanaka, S.~Niwa, Kinesin superfamily motor proteins and intracellular transport, Nat. Rev. Mol. Cell Biol. 10~(10) (2009) 682--696.
\newblock \href {https://doi.org/10.1038/nrm2774} {\path{doi:10.1038/nrm2774}}.

\bibitem{Ferrell2000}
J.~E. Ferrell, What do scaffold proteins really do?, Sci. STKE 2000~(52) (2000) pe1--pe1.
\newblock \href {https://doi.org/10.1126/stke.522000pe1} {\path{doi:10.1126/stke.522000pe1}}.

\bibitem{kim_intermediate_2007}
S.~Kim, P.~A. Coulombe, Intermediate filament scaffolds fulfill mechanical, organizational, and signaling functions in the cytoplasm, Genes Dev. 21~(13) (2007) 1581--1597.
\newblock \href {https://doi.org/10.1101/gad.1552107} {\path{doi:10.1101/gad.1552107}}.

\bibitem{Jumper_Alphafold2021}
J.~Jumper, R.~Evans, A.~Pritzel, et~al., Highly accurate protein structure prediction with {AlphaFold}, Nature 596~(7873) (2021) 583--589.
\newblock \href {https://doi.org/10.1038/s41586-021-03819-2} {\path{doi:10.1038/s41586-021-03819-2}}.

\bibitem{m_how_2023}
M.~Karelina, J.~J. Noh, R.~O. Dror, How accurately can one predict drug binding modes using alphafold models?, eLife (2023).
\newblock \href {https://doi.org/10.7554/elife.89386.1} {\path{doi:10.7554/elife.89386.1}}.

\bibitem{Brandes2021}
N.~Brandes, D.~Ofer, O.~Peleg, N.~Rappoport, Proteinbert: A universal deep-learning model of protein sequence and function, Bioinformatics 38~(8) (2021) 2102--2109.
\newblock \href {https://doi.org/10.1093/bioinformatics/btac020} {\path{doi:10.1093/bioinformatics/btac020}}.

\bibitem{ESM_Rives2021}
A.~Rives, J.~Meier, T.~Sercu, S.~Goyal, Z.~Lin, J.~Liu, D.~Guo, M.~Ott, C.~L. Zitnick, J.~Ma, R.~Fergus, Biological structure and function emerge from scaling unsupervised learning to 250 million protein sequences, Proc. Natl. Acad. Sci. U.S.A. 118~(15) (2021) e2016239118.
\newblock \href {https://doi.org/10.1073/pnas.2016239118} {\path{doi:10.1073/pnas.2016239118}}.

\bibitem{elnaggar2021prottrans}
A.~Elnaggar, M.~Heinzinger, C.~Dallago, G.~Rehawi, Y.~Wang, L.~Jones, T.~Gibbs, T.~Feher, C.~Angerer, M.~Steinegger, D.~Bhowmik, B.~Rost, Prottrans: Toward understanding the language of life through self-supervised learning, IEEE Trans. Pattern Anal. Mach. Intell. 44~(10) (2022) 7112--7127.
\newblock \href {https://doi.org/10.1109/TPAMI.2021.3095381} {\path{doi:10.1109/TPAMI.2021.3095381}}.

\bibitem{blastknn_cafa}
P.~Radivojac, W.~T. Clark, T.~R. Oron, et~al., A large-scale evaluation of computational protein function prediction, Nat. Methods 10~(3) (2013) 221--227.
\newblock \href {https://doi.org/10.1038/nmeth.2340} {\path{doi:10.1038/nmeth.2340}}.

\bibitem{diamond}
B.~Buchfink, C.~Xie, D.~H. Huson, Fast and sensitive protein alignment using {DIAMOND}, Nat. Methods 12~(1) (2015) 59--60.
\newblock \href {https://doi.org/10.1038/nmeth.3176} {\path{doi:10.1038/nmeth.3176}}.

\bibitem{gvp_Jing2021}
B.~Jing, S.~Eismann, P.~Suriana, R.~J.~L. Townshend, R.~Dror, \href{https://openreview.net/forum?id=1YLJDvSx6J4}{Learning from protein structure with geometric vector perceptrons}, in: The Ninth International Conference on Learning Representations (ICLR), 2021.
\newline\urlprefix\url{https://openreview.net/forum?id=1YLJDvSx6J4}

\bibitem{zhang2022gearnet}
Z.~Zhang, M.~Xu, A.~R. Jamasb, V.~Chenthamarakshan, A.~Lozano, P.~Das, J.~Tang, \href{https://openreview.net/forum?id=to3qCB3tOh9}{Protein representation learning by geometric structure pretraining}, in: The Eleventh International Conference on Learning Representations (ICLR), 2023.
\newline\urlprefix\url{https://openreview.net/forum?id=to3qCB3tOh9}

\bibitem{Wang2022LMGVP}
Z.~Wang, S.~A. Combs, R.~Brand, M.~R. Calvo, P.~Xu, G.~Price, N.~Golovach, E.~O. Salawu, C.~J. Wise, S.~P. Ponnapalli, P.~M. Clark, {LM}-{GVP}: {An} extensible sequence and structure informed deep learning framework for protein property prediction, Sci. Rep. 12~(1) (2022) 6832.
\newblock \href {https://doi.org/10.1038/s41598-022-10775-y} {\path{doi:10.1038/s41598-022-10775-y}}.

\bibitem{zhang2023systematic}
Z.~Zhang, C.~Wang, M.~Xu, V.~Chenthamarakshan, A.~Lozano, P.~Das, J.~Tang, \href{https://arxiv.org/abs/2303.06275}{A systematic study of joint representation learning on protein sequences and structures} (2023).
\newblock \href {http://arxiv.org/abs/2303.06275} {\path{arXiv:2303.06275}}.
\newline\urlprefix\url{https://arxiv.org/abs/2303.06275}

\bibitem{Modalbias}
Y.~Zhang, P.~E. Latham, A.~Saxe, Understanding unimodal bias in multimodal deep linear networks, in: Proceedings of the 41st International Conference on Machine Learning, ICML'24, JMLR.org, 2024.
\newblock \href {https://doi.org/10.5555/3692070.3694511} {\path{doi:10.5555/3692070.3694511}}.

\bibitem{SPLM}
D.~Wang, M.~Pourmirzaei, U.~L. Abbas, S.~Zeng, N.~Manshour, F.~Esmaili, B.~Poudel, Y.~Jiang, Q.~Shao, J.~Chen, D.~Xu, S-plm: Structure-aware protein language model via contrastive learning between sequence and structure, Adv. Sci. 12~(5) (2025) 2404212.
\newblock \href {https://doi.org/10.1002/advs.202404212} {\path{doi:10.1002/advs.202404212}}.

\bibitem{CLIP}
A.~Radford, J.~W. Kim, C.~Hallacy, A.~Ramesh, G.~Goh, S.~Agarwal, G.~Sastry, A.~Askell, P.~Mishkin, J.~Clark, G.~Krueger, I.~Sutskever, \href{https://proceedings.mlr.press/v139/radford21a.html}{Learning transferable visual models from natural language supervision}, in: Proceedings of the 38th {International} {Conference} on {Machine} {Learning}, PMLR, 2021, pp. 8748--8763.
\newline\urlprefix\url{https://proceedings.mlr.press/v139/radford21a.html}

\bibitem{10.7554/eLife.75751}
D.~del Alamo, D.~Sala, H.~S. Mchaourab, J.~Meiler, Sampling alternative conformational states of transporters and receptors with alphafold2, eLife 11 (2022) e75751.
\newblock \href {https://doi.org/10.7554/eLife.75751} {\path{doi:10.7554/eLife.75751}}.

\bibitem{10.1093/nargab/lqac043}
M.~Heinzinger, M.~Littmann, I.~Sillitoe, N.~Bordin, C.~Orengo, B.~Rost, Contrastive learning on protein embeddings enlightens midnight zone, NAR Genom. Bioinform. 4~(2) (2022) lqac043.
\newblock \href {https://doi.org/10.1093/nargab/lqac043} {\path{doi:10.1093/nargab/lqac043}}.

\bibitem{Szklarczyk2023_STRING}
D.~Szklarczyk, R.~Kirsch, M.~Koutrouli, et~al., The string database in 2023: protein–protein association networks and functional enrichment analyses for any sequenced genome of interest, Nucleic Acids Res. 51~(D1) (2022) D638--D646.
\newblock \href {https://doi.org/10.1093/nar/gkac1000} {\path{doi:10.1093/nar/gkac1000}}.

\bibitem{GeneOntologyConsortium2023}
T.~G.~O. Consortium, The gene ontology resource: enriching a gold mine, Nucleic Acids Res. 49~(D1) (2020) D325--D334.
\newblock \href {https://doi.org/10.1093/nar/gkaa1113} {\path{doi:10.1093/nar/gkaa1113}}.

\bibitem{hou2021tale}
J.~Hou, T.~Wu, R.~Cao, J.~Cheng, Tale: Transformer-based protein function annotation through joint sequence–label embedding, Bioinformatics 37~(18) (2021) 2825--2833.
\newblock \href {https://doi.org/10.1093/bioinformatics/btab198} {\path{doi:10.1093/bioinformatics/btab198}}.

\bibitem{DeepGO_2017}
M.~Kulmanov, M.~A. Khan, R.~Hoehndorf, Deepgo: predicting protein functions from sequence and interactions using a deep ontology-aware classifier, Bioinformatics 34~(4) (2017) 660--668.
\newblock \href {https://doi.org/10.1093/bioinformatics/btx624} {\path{doi:10.1093/bioinformatics/btx624}}.

\bibitem{DEEPGRAPHGO}
R.~You, S.~Yao, H.~Mamitsuka, S.~Zhu, Deepgraphgo: graph neural network for large-scale, multispecies protein function prediction, Bioinformatics 37~(Supplement\_1) (2021) i262--i271.
\newblock \href {https://doi.org/10.1093/bioinformatics/btab270} {\path{doi:10.1093/bioinformatics/btab270}}.

\bibitem{DPFunc_NatCommun2024}
W.~Wang, Y.~Shuai, M.~Zeng, W.~Fan, M.~Li, {DPFunc}: accurately predicting protein function via deep learning with domain-guided structure information, Nat. Commun. 16~(1) (2024) 70.
\newblock \href {https://doi.org/10.1038/s41467-024-54816-8} {\path{doi:10.1038/s41467-024-54816-8}}.

\bibitem{PPINosiy}
F.~B. Correia, E.~D. Coelho, J.~L. Oliveira, J.~P. Arrais, Handling noise in protein interaction networks, Biomed Res. Int. 2019~(1) (2019) 8984248.
\newblock \href {https://doi.org/10.1155/2019/8984248} {\path{doi:10.1155/2019/8984248}}.

\bibitem{pmlr-v202-dong23a}
M.~Dong, Y.~Kluger, \href{https://proceedings.mlr.press/v202/dong23a.html}{Towards understanding and reducing graph structural noise for {GNN}s}, in: A.~Krause, E.~Brunskill, K.~Cho, B.~Engelhardt, S.~Sabato, J.~Scarlett (Eds.), Proceedings of the 40th International Conference on Machine Learning, Vol. 202 of Proceedings of Machine Learning Research, PMLR, 2023, pp. 8202--8226.
\newline\urlprefix\url{https://proceedings.mlr.press/v202/dong23a.html}

\bibitem{Graffe}
D.~Chen, S.~Xue, L.~Chen, Y.~Wang, Q.~Liu, S.~Wu, Z.-M. Ma, L.~Wang, \href{https://arxiv.org/abs/2505.04956}{Graffe: Graph representation learning via diffusion probabilistic models} (2025).
\newblock \href {http://arxiv.org/abs/2505.04956} {\path{arXiv:2505.04956}}.
\newline\urlprefix\url{https://arxiv.org/abs/2505.04956}

\bibitem{CNN_Shanehsazzadeh}
A.~Shanehsazzadeh, D.~Belanger, D.~Dohan, \href{https://arxiv.org/abs/2011.03443}{Is transfer learning necessary for protein landscape prediction?} (2020).
\newblock \href {http://arxiv.org/abs/2011.03443} {\path{arXiv:2011.03443}}.
\newline\urlprefix\url{https://arxiv.org/abs/2011.03443}

\bibitem{LSTM_wang2016protein}
J.~Cheng, Y.~Liu, Y.~Ma, Protein secondary structure prediction based on integration of {CNN} and {LSTM} model, J. Visual Commun. Image Represent. 71 (2020) 102844.
\newblock \href {https://doi.org/10.1016/j.jvcir.2020.102844} {\path{doi:10.1016/j.jvcir.2020.102844}}.

\bibitem{deepfri_gligorijevic2021structure}
V.~Gligorijević, P.~D. Renfrew, T.~Kosciolek, J.~K. Leman, D.~Berenberg, T.~Vatanen, C.~Chandler, B.~C. Taylor, I.~M. Fisk, H.~Vlamakis, R.~J. Xavier, R.~Knight, K.~Cho, R.~Bonneau, Structure-based protein function prediction using graph convolutional networks, Nat. Commun. 12~(1) (2021) 3168.
\newblock \href {https://doi.org/10.1038/s41467-021-23303-9} {\path{doi:10.1038/s41467-021-23303-9}}.

\bibitem{hermosilla2021ieconv}
P.~Hermosilla, M.~Sch{\"a}fer, M.~Lang, G.~Fackelmann, P.-P. V{\'a}zquez, B.~Kozlikova, M.~Krone, T.~Ritschel, T.~Ropinski, \href{https://openreview.net/forum?id=l0mSUROpwY}{Intrinsic-extrinsic convolution and pooling for learning on 3d protein structures}, in: The Ninth International Conference on Learning Representations (ICLR), 2021.
\newline\urlprefix\url{https://openreview.net/forum?id=l0mSUROpwY}

\bibitem{dmasif_sverrisson2021fast}
F.~Sverrisson, J.~Feydy, B.~Correia, et~al., Fast end-to-end learning on protein surfaces, in: Proceedings of the IEEE/CVF Conference on Computer Vision and Pattern Recognition (CVPR), 2021, pp. 15272--15281.

\bibitem{hermosilla2022contrastive}
P.~Hermosilla, T.~Ropinski, \href{https://arxiv.org/abs/2205.15675}{Contrastive representation learning for 3d protein structures} (2022).
\newblock \href {http://arxiv.org/abs/2205.15675} {\path{arXiv:2205.15675}}.
\newline\urlprefix\url{https://arxiv.org/abs/2205.15675}

\bibitem{xu2023pretraining}
Z.~Zhang, M.~Xu, A.~Lozano, V.~Chenthamarakshan, P.~Das, J.~Tang, \href{https://openreview.net/forum?id=tIzbNQko3c}{Pre-training protein encoder via siamese sequence-structure diffusion trajectory prediction}, in: Advances in Neural Information Processing Systems, 2023.
\newline\urlprefix\url{https://openreview.net/forum?id=tIzbNQko3c}

\bibitem{cdconv_Fan2023}
H.~Fan, Z.~Wang, Y.~Yang, M.~Kankanhalli, \href{https://openreview.net/forum?id=P5Z-Zl9XJ7}{Continuous-discrete convolution for geometry-sequence modeling in proteins}, in: The Eleventh International Conference on Learning Representations (ICLR), 2023.
\newline\urlprefix\url{https://openreview.net/forum?id=P5Z-Zl9XJ7}

\bibitem{zhang2023sst}
G.~Zhou, Y.~Zhao, S.~He, X.~Bo, {SST}-{ResNet}: {A} sequence and structure information integration model for protein property prediction, Int. J. Mol. Sci. 26~(6) (2025) 2783.
\newblock \href {https://doi.org/10.3390/ijms26062783} {\path{doi:10.3390/ijms26062783}}.

\bibitem{zhang2022atgo}
Y.-H. Zhu, C.~Zhang, D.-J. Yu, Y.~Zhang, Integrating unsupervised language model with triplet neural networks for protein gene ontology prediction, PLoS Comput Biol 18~(12) (2022) e1010793.
\newblock \href {https://doi.org/10.1371/journal.pcbi.1010793} {\path{doi:10.1371/journal.pcbi.1010793}}.

\bibitem{poincare}
M.~Nickel, D.~Kiela, \href{https://proceedings.neurips.cc/paper_files/paper/2017/file/59dfa2df42d9e3d41f5b02bfc32229dd-Paper.pdf}{Poincaré embeddings for learning hierarchical representations}, in: I.~Guyon, U.~V. Luxburg, S.~Bengio, H.~Wallach, R.~Fergus, S.~Vishwanathan, R.~Garnett (Eds.), Advances in {Neural} {Information} {Processing} {Systems}, Vol.~30, Curran Associates, Inc., 2017.
\newline\urlprefix\url{https://proceedings.neurips.cc/paper_files/paper/2017/file/59dfa2df42d9e3d41f5b02bfc32229dd-Paper.pdf}

\bibitem{Cuturi2013_Sinkhorn}
M.~Cuturi, Sinkhorn distances: Lightspeed computation of optimal transport, in: Advances in Neural Information Processing Systems, 2013, pp. 2292--2300.

\bibitem{DiGress2022}
C.~Vignac, I.~Krawczuk, A.~Siraudin, B.~Wang, V.~Cevher, P.~Frossard, \href{https://openreview.net/forum?id=UaAD-Nu86WX}{Digress: Discrete denoising diffusion for graph generation}, in: The Eleventh International Conference on Learning Representations (ICLR), 2022.
\newline\urlprefix\url{https://openreview.net/forum?id=UaAD-Nu86WX}

\bibitem{diffusion_cosine}
P.~Dhariwal, A.~Q. Nichol, \href{https://openreview.net/forum?id=AAWuCvzaVt}{Diffusion models beat {GAN}s on image synthesis}, in: A.~Beygelzimer, Y.~Dauphin, P.~Liang, J.~W. Vaughan (Eds.), Advances in Neural Information Processing Systems, Vol.~35, 2021.
\newline\urlprefix\url{https://openreview.net/forum?id=AAWuCvzaVt}

\bibitem{ho2021_CFG}
J.~Ho, T.~Salimans, \href{https://openreview.net/pdf?id=qw8AKxfYbI}{Classifier-free diffusion guidance}, in: NeurIPS 2021 Workshop on Deep Generative Models and Downstream Applications (NeurIPS Workshop), 2021.
\newline\urlprefix\url{https://openreview.net/pdf?id=qw8AKxfYbI}

\bibitem{graphTransfomre}
V.~P. Dwivedi, X.~Bresson, \href{https://arxiv.org/abs/2012.09699}{A generalization of transformer networks to graphs}, in: AAAI Workshop on Deep Learning on Graphs: Methods and Applications, 2021.
\newline\urlprefix\url{https://arxiv.org/abs/2012.09699}

\bibitem{GraphSAGE}
W.~Hamilton, Z.~Ying, J.~Leskovec, \href{https://proceedings.neurips.cc/paper_files/paper/2017/file/5dd9db5e033da9c6fb5ba83c7a7ebea9-Paper.pdf}{Inductive representation learning on large graphs}, in: I.~Guyon, U.~V. Luxburg, S.~Bengio, H.~Wallach, R.~Fergus, S.~Vishwanathan, R.~Garnett (Eds.), Advances in {Neural} {Information} {Processing} {Systems}, Vol.~30, Curran Associates, Inc., 2017.
\newline\urlprefix\url{https://proceedings.neurips.cc/paper_files/paper/2017/file/5dd9db5e033da9c6fb5ba83c7a7ebea9-Paper.pdf}

\bibitem{kulmanov_deepgozero_2022}
M.~Kulmanov, R.~Hoehndorf, {DeepGOZero}: improving protein function prediction from sequence and zero-shot learning based on ontology axioms, Bioinformatics 38 (2022) i238--i245.
\newblock \href {https://doi.org/10.1093/bioinformatics/btac256} {\path{doi:10.1093/bioinformatics/btac256}}.

\bibitem{deepgoplus_Kulmanov2019}
M.~Kulmanov, R.~Hoehndorf, {DeepGOPlus}: improved protein function prediction from sequence, Bioinformatics 36~(2) (2020) 422--429.
\newblock \href {https://doi.org/10.1093/bioinformatics/btz595} {\path{doi:10.1093/bioinformatics/btz595}}.

\bibitem{adamw}
I.~Loshchilov, F.~Hutter, \href{https://openreview.net/forum?id=Bkg6RiCqY7}{Decoupled weight decay regularization}, in: The Seventh International Conference on Learning Representations (ICLR), 2019.
\newline\urlprefix\url{https://openreview.net/forum?id=Bkg6RiCqY7}

\bibitem{DeepGS}
Q.~Yuan, J.~Chen, H.~Zhao, Y.~Zhou, Y.~Yang, Structure-aware protein–protein interaction site prediction using deep graph convolutional network, Bioinformatics 38~(1) (2021) 125--132.
\newblock \href {https://doi.org/10.1093/bioinformatics/btab643} {\path{doi:10.1093/bioinformatics/btab643}}.

\end{thebibliography}

\end{document}